\RequirePackage[l2tabu, orthodox]{nag}
\documentclass[12pt]{article}

\pdfoutput=1

\usepackage[usenames,dvipsnames]{xcolor}
\usepackage[margin=1.0in]{geometry}
\usepackage{graphicx}

\usepackage[sc]{mathpazo}
\usepackage[scaled=.95]{helvet}
\usepackage{courier}

\usepackage[english]{babel}
\usepackage[utf8]{inputenc}
\usepackage[T1]{fontenc}
\usepackage{microtype}

\usepackage{amsfonts}
\usepackage{amsthm}
\usepackage{amsmath}
\usepackage{amssymb}

\usepackage{xspace}

\usepackage[sectionbib]{natbib}
\usepackage[nottoc]{tocbibind}
\usepackage{breakcites}

\usepackage{enumerate}
\usepackage{enumitem}
\usepackage{array}

\usepackage{booktabs}
\usepackage{chngpage}
\usepackage{multirow}

\usepackage[pdfencoding=auto,bookmarks=false,colorlinks=true,linkcolor=blue,urlcolor=blue,citecolor=blue,pdfstartview=FitH,pagebackref,backref=false]{hyperref}

\usepackage{thmtools}
\usepackage{thm-restate}

\usepackage[capitalize]{cleveref}

\declaretheorem[name=Theorem,numberwithin=section]{theorem}

\declaretheorem[name=Lemma,sibling=theorem]{lemma}

\declaretheorem[name=Assumption]{assumption}
\Crefname{assumption}{Assumption}{Assumptions}

\declaretheorem[name=Definition,sibling=theorem]{definition}

\declaretheorem[name=Conjecture,sibling=theorem]{conjecture}

\declaretheorem[name=Condition]{condition}
\Crefname{condition}{Condition}{Conditions}

\declaretheorem[name=Fact,sibling=theorem]{fact}
\Crefname{fact}{Fact}{Facts}

\numberwithin{equation}{section}

\DeclareMathOperator*{\argmax}{argmax}
\DeclareMathOperator*{\argmin}{argmin}

\DeclareMathOperator{\Bern}{Bernoulli}

\newcommand{\ie}{\emph{i.e.\@\xspace}}
\newcommand{\eg}{\emph{e.g.\@\xspace}}
\newcommand{\st}{s.t.\@\xspace}
\newcommand{\wrt}{w.r.t.\@\xspace}
\newcommand{\wpr}{w.p.\@\xspace}

\newcommand{\aev}{a.e.\@\xspace}

\newcommand{\iid}{i.i.d.\@\xspace}
\newcommand{\cf}{cf.\@\xspace}
\newcommand{\see}{see\@\xspace}
\newcommand{\seea}{see also\@\xspace}

\newcommand{\viz}{viz.\@\xspace}

\newcommand{\lrblock}[3]{\left#1 #2 \right#3}

\newcommand{\cb}[1]{\lrblock{\{}{#1}{\}}}
\newcommand{\pb}[1]{\lrblock{(}{#1}{)}}
\newcommand{\brb}[1]{\lrblock{[}{#1}{]}}
\newcommand{\ab}[1]{\lrblock{|}{#1}{|}}

\newcommand{\range}[1]{\brb{#1}}

\newcommand{\eps}{\varepsilon}
\newcommand{\ind}[1]{\mathbb{I}\cb{#1}}

\newcommand{\fun}[3]{#1 : #2 \rightarrow #3}
\newcommand{\rl}[1]{\mathbb{R}^{#1}}
\newcommand{\tra}{^\top}
\newcommand{\ra}{\rightarrow}

\newcommand{\one}[1]{\mathbf{1}_{#1}}
\newcommand{\zero}[1]{\mathbf{0}_{#1}}

\newcommand{\jeps}{j_{\eps}}
\newcommand{\zo}{$0$-$1$}
\newcommand{\eqdef}{\doteq}

\newcommand{\setdiff}{\backslash}
\newcommand{\K}{\ab{\sn{Y}}}

\newcommand{\cond}[2]{\left. #1 \right | #2}
\newcommand{\muE}[2]{\mathbb{E}_{#1}\brb{ #2 }}
\newcommand{\cE}[2]{ \muE{}{ \cond{#1}{#2} } }
\newcommand{\E}[1]{\muE{}{#1}}
\newcommand{\Prob}[1]{\mathbb{P}\cb{#1}}
\newcommand{\cProb}[2]{\Prob{ \cond{#1}{#2} }}

\newcommand{\sn}[1]{\mathcal{#1}} 

\newcommand{\simplex}[1]{\Delta_{#1}}
\newcommand{\yplex}{\simplex{\ab{\sn{Y}}}}

\newcommand{\al}[1]{\begin{align}#1\end{align}}
\newcommand{\als}[1]{\begin{align*}#1\end{align*}}
\newcommand{\eq}[1]{\begin{equation}#1\end{equation}}
\newcommand{\eqs}[1]{\begin{equation*}#1\end{equation*}}

\newcommand{\risk}{R}

\newcommand{\surr}{\mathrm{surr}}

\newcommand{\sinf}{t^{\inf}}
\newcommand{\A}{\mathrm{A}}

\newcommand{\eye}{\varphi^{\mathrm{identity}}}
\newcommand{\hinge}{\varphi^{\mathrm{hinge}}}
\newcommand{\linear}{\varphi^{\mathrm{linear}}}
\newcommand{\squared}{\varphi^{\mathrm{squared}}}
\newcommand{\trunc}{\varphi^{\mathrm{trunc-sq}}}
\newcommand{\expt}{\varphi^{\mathrm{exp}}}
\newcommand{\logit}{\varphi^{\mathrm{logit}}}
\newcommand{\sigmoid}{\varphi^{\mathrm{sigmoid}}}
\newcommand{\modulus}{\varphi^{\mathrm{modulus}}}
\newcommand{\zot}{\varphi^{\mathrm{0-1}}}

\newcommand{\kink}{\varphi^{\tau\mathrm{-kink}}}

\newcommand{\sloss}[1]{L^{\mathrm{#1}}}

\newcommand{\adj}{l}

\title{Multiclass Classification Calibration Functions}

\author{
    \begin{tabular}{cc}
        Bernardo \'Avila Pires & Csaba Szepesv\'ari \\
        \url{bpires@ualberta.ca} & \url{szepesva@cs.ualberta.ca} \\
        \multicolumn{2}{c}{Department of Computing Science} \\
        \multicolumn{2}{c}{University of Alberta} \\
        \multicolumn{2}{c}{Edmonton, Alberta, Canada}
    \end{tabular}            
}

\date{\today}

\begin{document}

\maketitle

\begin{abstract}
In this paper we refine the process of computing calibration functions for a number of multiclass classification surrogate losses.
Calibration functions are a powerful tool for easily converting bounds for the surrogate risk (which can be computed through well-known methods) into bounds for the true risk, the probability of making a mistake.
They are particularly suitable in non-parametric settings, where the approximation error can be controlled, and provide tighter bounds than the common technique of upper-bounding the \zo{} loss by the surrogate loss.

The abstract nature of the more sophisticated existing calibration function results requires calibration functions to be explicitly derived on a case-by-case basis, requiring repeated efforts whenever bounds for a new surrogate loss are required.
We devise a streamlined analysis that simplifies the process of deriving calibration functions for a large number of surrogate losses that have been proposed in the literature.
The effort of deriving calibration functions is then surmised in verifying, for a chosen surrogate loss, a small number of conditions that we introduce.

As case studies, we recover existing calibration functions for the well-known loss of \citet{lee2004multicategory}, and also provide novel calibration functions for well-known losses, including the one-versus-all loss and the logistic regression loss, plus a number of other losses that have been shown to be classification-calibrated in the past, but for which no calibration function had been derived.
\end{abstract}

\section{Introduction}
\label{sec:intro}

\emph{Classification} is a well-studied discrete prediction problem, and in this paper we are interested in risk bounds for classifiers.
The classifiers we focus on are obtained through Empirical Risk Minimization \citep[ERM,][]{vapnik2013nature}.
While the goal in classification is to minimize the misclassification probability (a.k.a.~the expected \zo{} loss or the \emph{risk}), 
minimizing the empirical risk (the \zo{} loss on a sample, as prescribed by ERM) can be computationally hard \citep{hoffgen1995robust}.
So it is common to minimize a convex \emph{surrogate} loss, as a means to minimize the true risk, which is defined in terms of the non-convex \zo{} loss.
ERM with the surrogate loss gives us approximate minimizers of the \emph{surrogate} risk, \ie{}, the expected surrogate loss, so the first question that comes to mind is whether minimizers of the surrogate risk are also minimizers of the true risk, \ie{}, whether, ultimately, good classifiers can be obtained by ERM with a surrogate loss.
When a surrogate loss enjoys this guarantee, we say that it is \emph{calibrated} or \emph{Fisher-consistent}.
The question of calibration for different losses has been recently investigated by a number of authors \citep{zhang2004statistical,liu2007fisher,tewari2007ontheconsistency,reid2009surrogate,guruprasad2012classification,ramaswamy2013convex,calauzenes2013calibration,ramaswamy2016convex,dougan2016unified}.
Rather than being concerned with just calibration, here we investigate how to obtain bounds for the true risk of (surrogate-risk) ERM classifiers.
Thanks to the solid understanding of techniques to obtain risk bounds for empirical risk minimizers \citep{steinwart2008support,koltchinskii2011oracle}, we can follow the approach of \citet{steinwart2007how,avilapires2013costsensitive} and focus on techniques for \emph{converting} surrogate risk bounds into true risk bounds (to which we will henceforth refer as ``bound conversion'').
As a bonus, having effective means to perform this conversion also enables us to answer calibration questions for surrogate losses.

We build on the works of \citet{steinwart2007how,avilapires2013costsensitive}, exploring the concept of \emph{calibration functions}, which are an effective tool for bound conversion.
In fact, the toolset developed by \citet{steinwart2007how} fully constitutes an approach for bound conversion, at least in an abstract sense.
This toolset generalizes techniques for the binary case that were introduced by \citet{bartlett2006convexity} and used to characterize bound conversion for a large family of popular surrogate losses.
Unfortunately, notwithstanding their power, the techniques of \citet{steinwart2007how} are too abstract for one to perform bound conversion for \emph{multiclass} losses without a significant amount of effort directed to each specific loss, an effort surmised in the calculation of the aforementioned calibration functions.
In contrast, calibration functions for common choices of \emph{binary} surrogate losses can be obtained almost immediately (\see{} \cref{thm:bartlett}).

Our goal is, therefore, to simplify the process of calculating calibration functions in the multiclass case for various surrogate losses, and our main contribution is a generic way to ``reduce'' multiclass calibration functions to binary ``calibration-like'' functions that can be easily computed for specific surrogate loss choices.

We achieve our goal by designing a set of conditions that, when satisfied by a particular loss, yield a function that is essentially a calibration function for a binary loss, similar to the calibration function presented by \citet{bartlett2006convexity} for margin-based losses.
As an advantage, we are able easily generalize, to the multiclass case, a result by \citet{bartlett2006convexity} that gives improved calibration functions when the distribution of $(X,Y)$ satisfies the Mammen-Tsybakov noise condition \citep{mammen1999smooth,boucheron2005theory,bartlett2006convexity}.

Our analysis generalizes the work of \citet{avilapires2013costsensitive}, who presented calibration functions for a family of multiclass classification losses introduced by \citet{lee2004multicategory}.
While \citet{avilapires2013costsensitive} investigate a cost-sensitive setting, we restrict our considerations to the ordinary (cost-insensitive) classification problem, and we refine their results for this case.

While various multiclass surrogate losses have been proposed in the literature (\see{} \cref{table:scoreLosses}), many of them share a similar structure that allows our results to be widely applicable.
In order to illustrate the application of our results, we perform case studies for our analysis.
We verify the proposed conditions in order to easily obtain calibration functions for the loss of \citet{lee2004multicategory}, thus recovering the results of \citet{avilapires2013costsensitive} in the cost-insensitive setting.
Also as case studies, we obtain novel calibration functions for the decoupled unconstrained background discrimination losses presented by \citet{zhang2004statistical}, the logistic regression loss (a special case of the \emph{coupled} unconstrained background discrimination losses), and the one-versus-all loss \citep{rifkin2004indefense}.
Specific instantiations of the decoupled unconstrained background discrimination loss (including the one-versus-all loss) require verification of an additional condition that, we believe, is not harder to verify than a related binary calibration function is to derive.
We verify this condition for some choices of unconstrained background discrimination losses.
Our analysis does not cover the surrogate losses proposed by \citet{weston1998multiclass,zou2006margin,beijbom2014guessaverse}, for which we believe a different analysis is required.

This work is structured as follows.
In \cref{sec:preliminaries}, we introduce the classification problem and some notation, and we discuss the conversion of surrogate risk bounds into true risk bounds.
\Cref{sec:calibrationToolset} presents a review of related work and introduces the core concepts that we use for the bound conversion.
We follow with \cref{sec:streamlining}, where we introduce a general analysis that allows us to reduce multiclass calibration functions to binary calibration functions.
Then, in \cref{sec:cases}, we perform ``case studies'' by looking at how the analysis works for specific families of surrogate losses.
We conclude this work in \cref{sec:conclusion}, with a commentary on the strengths and limitations of our results, and a discussion of possible extensions of our work.

\section{Preliminaries}
\label{sec:preliminaries}
\paragraph{Problem definition.}
In classification we wish to find a function\footnote{
	We will frequently omit well-understood technical details, such as measurability.
	In our discussions (but not the proofs), we will mention minimizers of lower-bounded functions that may not have a minimizer, \eg{}, the exponential function.
	In those cases, the considerations are easily extended to approximate minimizers that are arbitrarily close to the infimum.
} $\fun{g}{\sn{X}}{\sn{Y}}$, called a \emph{classifier}, achieving the smallest \emph{expected misclassification error}, also known as the \emph{misclassification rate} or \emph{risk}
\eq{
	\label[equation]{eq:risk}
	\risk(g) \eqdef \Prob{g(X) \neq Y},
}
where $(X,Y) \sim p$ are jointly distributed random variables taking values in sets $\sn{X}$ and $\sn{Y} \eqdef \range{\K} \eqdef \cb{ 1, \dots, \K }$, respectively.
The risk can also be written as the expected value of a $\emph{loss}$, in this case the function $y, y' \mapsto \ind{y \neq y'}$, which is called the \emph{\zo{} loss}.
The goal of the classification problem can also be stated as minimizing the \emph{excess risk}
\eqs{
	\label[equation]{eq:excessRisk}
	\risk(g) - \inf_{g'} \risk(g').
}
whenever the \emph{Bayes-risk} $\inf_{g'} \risk(g')$ is bounded in absolute value.

What we defined as the classification problem is often referred to as \emph{multiclass classification}.
When $\K = 2$, in particular, the classification problem is called \emph{binary classification}.
It is possible to define the classification problem in more general terms, to include \emph{cost-sensitive classification} \citep{zhang2004statistical,steinwart2007how,ramaswamy2013convex,avilapires2013costsensitive}, however we will leave this direction aside in this work.

In the classification learning problem, the distribution $p$ is unknown and we are only given a finite, \iid{} sample $((X_1,Y_1),\ldots,(X_n,Y_n)) \sim p^n$.
Moreover, one typically fixes a set of classifiers, $\sn{G} \subset \sn{Y}^\sn{X}$, called the \emph{hypothesis class}, in which case the goal of the problem can be written as minimizing the \emph{$\sn{G}$-excess risk}
\[
	\risk(g) - \inf_{g' \in \sn{G}} \risk(g').
\]
whenever $\ab{\inf_{g' \in \sn{G}} \risk(g')} < \infty$.

\paragraph{Empirical risk minimization.}
ERM, a common approach for solving classification problems (\citealp[p.~8]{steinwart2008support}; \citealp[p.~15]{shalev2014understanding}), prescribes that we solve
\[
	\min_{g \in \sn{G}}\risk_n(g),
\]
where
\eq{
	\risk_n(g) \eqdef \frac{1}{n}\sum_{i=1}^n \ind{ g(X_i) \neq Y_i }\label[equation]{eq:empiricalRisk}
}
is called the \emph{empirical (\zo{}) risk}.

\paragraph{Surrogate losses.}
As shown by \citep{hoffgen1995robust,ben2003difficulty,feldman2012agnostic,nguyen2013algorithms}, computing empirical risk minimizers for the empirical risk in \eqref{eq:empiricalRisk} is $\mathcal{NP}$-hard for some commonly used hypothesis classes, so one often replaces the empirical risk with an \emph{empirical surrogate risk}
\eq{
	\risk^{\surr}_n(h) \eqdef \frac{1}{n}\sum_{i=1}^n L(h(X_i), Y_i), \label[equation]{eq:empiricalSurrogateRisk}
}
where $\fun{L}{\sn{S} \times \sn{Y}}{\rl{}}$ is a convex \emph{surrogate} loss, $\sn{S} \subset \rl{\K}$ is a \emph{set of scores} and $h \subset (\rl{\K})^{\sn{X}}$ is a \emph{score function}.
One chooses the loss $L$ and the hypothesis class $\sn{H} \subset (\rl{\K})^{\sn{X}}$ so that minimizing \eqref{eq:empiricalSurrogateRisk} over $h \in \sn{H}$ can be done efficiently.
ERM with \eqref{eq:empiricalSurrogateRisk} as its objective will allow us to obtain guarantees (risk bounds) for the surrogate risk
\[
	\risk^{\surr}(h) \eqdef \E{L(h(X), Y)}.
\]

While the true loss yields a value in $\cb{0,1}$ when given a prediction $y' \in \sn{Y}$ and a class $y \in \sn{Y}$, the surrogate loss will yield a real number when given a $\K$-dimensional real vector $s$, called a \emph{score}, and a class $y \in \sn{Y}$.
The set of scores $\sn{S}$ is commonly chosen to be $\rl{\K}$ itself, or the space of \emph{sum-to-zero} scores $\sn{S}_0 \eqdef \cb{s \in \rl{\K} : \one{\K} \tra s = 0}$, where $\one{\K}$ is the $\K$-dimensional vector of ones.
A third common choice is the \emph{$\K$-dimensional simplex} $\yplex$.

In order to properly have classifiers, we will transform scores into classes using the \emph{maximum selector} $\fun{f}{\sn{S}}{\sn{Y}}$ defined by\footnote{
	If the $\argmax$ is not a singleton we pick an arbitrary element from it.
	Our results will be worst-case when it comes to ties, so that tie-breaking in the maximum selector is not an issue and can be done arbitrarily.
}
\[
	f(s) \eqdef \argmax_y s_y.
\]
It is easy to show that any classifier can be obtained by composing a score function with the maximum selector, so using score functions does not inherently limit solutions for the classification problem.

In the binary case, a common loss choice is a \emph{margin loss} \citep[Section 2.3]{steinwart2008support} $L^{\varphi}(s,y) \eqdef \varphi(-s_y)$ with score set $\sn{S} = \sn{S}_0$ and convex \emph{transformation function} $\varphi$.
\Cref{table:varphi} contains some common choices of $\varphi$ (\seea{} \citealp[Section 2.3]{steinwart2008support}; \citealp[Chapter 4]{hastie2009elements}), and it includes non-convex choices for completeness.
\begin{table}[ht]
\begin{adjustwidth}{-1in}{-1in}
\centering
\begin{tabular}{l>{\arraybackslash}l}
\toprule
Transformation function & Definition \\
\midrule
Misclassification (\zo{}) & $\zot(t) \eqdef \ind{t \geq 0}$ \\
Identity & $\eye(t) \eqdef t$ \\
Linear & $\linear(t) \eqdef 1 + t $ \\
Hinge & $\hinge(t) \eqdef (1 + t)_+$ \\
Modulus & $\modulus(t) \eqdef \ab{1 + t}$ \\
Squared & $\squared(t) \eqdef (1 + t)^2$ \\
Truncated square (squared hinge) & $\trunc(t) \eqdef \hinge(t)^2 $ \\
Exponential & $\expt(t) \eqdef e^t$ \\
Logistic & $\logit(t) \eqdef \ln(1 + e^t)$ \\
Sigmoid & $\sigmoid(t) \eqdef \frac{1}{1 + e^{-t}}$ \\
Kink & $\kink(t) \eqdef \hinge(t) + (t - \tau)_+$ \\
\bottomrule
\end{tabular}
\end{adjustwidth}
\caption{Different choices of $\varphi$. \label{table:varphi}}
\end{table}
Additional choices can be found in the works of \citet{mason1999boosting,gneiting2007strictly,nock2009bregman,reid2010composite,shi2015hybrid}.

Applying ERM to some of the margin losses with $\varphi$ from \cref{table:varphi} and an appropriate choice of $\sn{H}$ in fact gives a correspondence to successful binary classification methods.
For example, SVMs use $\hinge$, Ridge regression uses $\squared$, Logistic regression uses $\logit$ and AdaBoost uses $\expt$ \citep[\see{} Table 21.1, and Sections 4.4.1 and 10.4 of][]{hastie2009elements}.

\paragraph{Bound conversion.}
The convexity of $L$ (precisely, the convexity of $s \mapsto L(s,y)$ for all $y \in \sn{Y}$) and the choice of $\sn{H}$ will ensure that the empirical risk can be minimized efficiently, which satisfactorily addresses the computational side to constructing classifiers with ERM and surrogate losses.
On the statistical side, we are concerned with obtaining true risk bounds for these classifiers.

We know that under certain conditions ERM will give us, with high probability, an approximate surrogate risk minimizer over $\sn{H}$ \citep{steinwart2008support,koltchinskii2011oracle}, which will be a bound on the surrogate risk of a score function.
We have a straightforward way to convert score functions into classifiers, so what remains is to show that ERM on the empirical surrogate risk will also give us, with high-probability, an approximate true risk minimizer.

While our work is primarily motivated by ERM, our main concern is bound conversion, and are be able to make statements about any \emph{learning algorithm} for which a surrogate risk bound is available.
A learning algorithm $\fun{A}{\bigcup_{n=1}^\infty (\sn{X} \times \sn{Y})^n}{\sn{X}^{\sn{Y}}}$ is a function mapping samples to hypotheses
\footnote{
	Extensions to randomized algorithms, mapping samples to distributions over score functions, are straightforward.
}.
For example, learning algorithms following the ERM approach satisfy
\eq{
	\A(S) \in \argmin_{h \in \sn{H}}\risk^{\surr}_n(h), \label[equation]{eq:hHat}
}
when given a sample $S \in (\sn{X} \times \sn{Y})^n$.

As we are not concerned with specifics of bounding the surrogate risk, we will leave this problem aside in our discussion of bound conversion.
\Cref{ass:surrogateRiskBound} establishes that classifiers constructed by a given learning algorithm $\A$ from a random sample of size $n$ have surrogate risk (conditioned on the sample) bounded by $T^{\surr}_{\sn{H}}(n,\delta)$ with probability at least $1 - \delta$.
The surrogate risk of a random score function $H$ conditioned on a sample $S$ taking values in $(\sn{X} \times \sn{Y})^n$ is defined as 
\[
	\risk^{\surr}(H, S) \eqdef \cE{H(h(X), Y)}{S}.
\]
Similarly, for a random classifier $G$,
\[
	\risk(G, S) \eqdef \cProb{G(X) \neq Y}{S}.
\]
\begin{assumption}[Surrogate risk bound]
\label{ass:surrogateRiskBound}
Given a learning algorithm $\fun{A}{\bigcup_{n=1}^\infty (\sn{X} \times \sn{Y})^n}{\sn{X}^{\sn{Y}}}$ and a hypothesis class $\sn{H} \subset (\rl{\K})^{\sn{X}}$, there exists a function $\fun{T^{\surr}_{\sn{H}}}{\mathbb{N} \times (0,1)}{\rl{}}$ \st{} for every $n \geq 1$ and $\delta \in (0,1)$ the following holds \wpr{} at least $1 - \delta$:
\eq{
	\label[equation]{eq:surrogateRiskBound}
	\risk^{\surr}(A(S), S) - \inf_{h \in \sn{H}}\risk^{\surr}(h) < T^{\surr}_{\sn{H}}(n,\delta),
}
where $S \in (\sn{X} \times \sn{Y})^n$ is a $P$-\iid{}
\end{assumption}

\citet{steinwart2008support,koltchinskii2011oracle} discuss techniques for obtaining bounds that satisfy \cref{ass:surrogateRiskBound}.
Alternatively, we can replace \cref{ass:surrogateRiskBound} with an assumption that an ``expectation bound'' is available, \ie{}, that \eqref{eq:surrogateRiskBound} holds in expectation (where the expectation is taken over the sample $S$).
In this case, using bound conversion we obtain expectation bounds for the true risk bounds.
\Cref{ass:surrogateRiskBound}, in addition to a $\sn{H}$-excess surrogate risk bound, also gives an excess surrogate risk bound, since \eqref{eq:surrogateRiskBound} implies that
\[
	\risk^{\surr}(\hat{h}) - \inf_{h}\risk^{\surr}(h) < T^{\surr}_{\sn{H}}(n,\delta) + A^\surr_\sn{H},
\]
where
\[
	A^\surr_\sn{H} \eqdef \inf_{h \in \sn{H}}\risk^{\surr}(h) - \inf_{h}\risk^{\surr}(h)
\]
is the \emph{approximation error} \wrt{} the surrogate risk, which can only be made small with appropriate choices of $\sn{H}$.
From a non-parametric point of view, one should control $A^\surr_\sn{H}$ by trading it off with $T^{\surr}_{\sn{H}}(n,\delta)$, so as to obtain an appropriate rate of convergence for the excess surrogate risk \citep[p.8]{steinwart2008support}.

There are different ways to convert surrogate risk bounds into true risk bounds.
The following well-known result can be applied if the surrogate loss upper-bounds the \zo{} loss.
\begin{theorem}[\citealt{boucheron2005theory}]
\label{thm:boucheron}
Given $\A$ and $\sn{H}$ satisfying \cref{ass:surrogateRiskBound}, if $L(s,y) \geq \ind{f(s) \neq y}$ ($s \in \sn{S}$, $y \in \sn{Y}$), then, for all $\delta \in (0,1)$, with probability at least $1 - \delta$, we have
\[
	\risk(f \circ \A(S), S) \leq T^{\surr}_{\sn{H}}(n,\delta) + \inf_{h \in \sn{H}}\risk^{\surr}(h).
\]
\end{theorem}

A limitation of \cref{thm:boucheron} is that the resulting true risk bounds can be loose \citep{lin2004note}.
For example, take $\sn{X} = \cb{x}$, $Y \sim \Bern(p)$ for $p \in [0,1]$, $\sn{S} = \sn{S}_0 \subset \rl{2}$, and $L(s,y) = (1 - s_y)_+$ (the hinge loss).
Then $\inf_{h \in \sn{H}} \risk(f \circ h) = \min\cb{p,1 - p}$, but $\inf_{h \in \sn{H}}\risk^{\surr}(h) = 2\min\cb{p,1 - p}$.
Besides the undesirable factor of $2$, even if $T^{\surr}_{\sn{H}}(n,\delta) \rightarrow 0$ as $n \rightarrow \infty$, we cannot guarantee \emph{from the bound in \cref{thm:boucheron}} that we get an optimal classifier with probability at least $1 - \delta$.

\citet{zhang2004statistical,lin2004note,chen2006consistency,bartlett2006convexity,steinwart2007how} provide tighter guarantees for the true risk by using what came to be known as \emph{calibration functions}.
The following theorem can be inferred from the more general theoretical framework proposed by \citet{steinwart2007how}.
\begin{theorem}[\citealt{steinwart2007how}]
\label{thm:steinwart}
Given a surrogate loss $L$, assume that there exists a positive function $\fun{\delta}{(0,\infty)}{(0, \infty]}$ \st{}, for every $\eps > 0$, $s \in \sn{S}$ and $p \in \yplex$, with $Y' \sim p$ if
\eq{
	\E{L(s,Y')} - \inf_{s' \in \sn{S}}\E{L(s',Y')} < \delta(\eps) \label[equation]{eq:surrRiskGuarantee}
}
then
\eq{
	\Prob{f(s) \neq Y'} - \min_y \Prob{y \neq Y'} < \eps. \label[equation]{eq:trueRiskGuarantee}
}

Assume also that $\cE{\inf_{s \in \sn{S}} L(s,Y)}{X}$ is measurable.
Then, given $\A$ and $\sn{H}$ satisfying \cref{ass:surrogateRiskBound}, for all $\delta' \in (0,1)$, \wpr{} at least $1 - \delta'$, we have
\[
	\risk(f \circ \A(S), S) - \inf_{g}\risk(g) \leq \delta^{-1}\pb{ T^{\surr}_{\sn{H}}(n,\delta') + A^\surr_\sn{H} }.
\]
\end{theorem}

The function $\delta$ in \cref{thm:steinwart} is called a \emph{calibration function} \citep{steinwart2007how}.
\citet{steinwart2007how} presents a general, extensive discussion on calibration functions, a few of which are reported in \cref{sec:calibrationFunctions}.
In order to properly obtain \cref{thm:steinwart} even if $\delta$ is not invertible, we can use $\delta^{-1}(x) \eqdef \inf\cb{\eps: \delta(\eps) \geq x}$.

\Cref{thm:steinwart} states that if the excess surrogate risk goes to zero at a particular rate, we can also get a rate at which the excess true risk goes to zero.
Sometimes, we may know that a calibration function exists for $L$ without knowing the calibration function itself.
In this case, we know that the the excess surrogate risk goes to zero iff the excess true risk goes to zero, so surrogate risk minimizers are also true risk minimizers.
Conversely, if some surrogate risk minimizer is not a true risk minimizer, then no calibration function can exist.
The existence of a calibration function is equivalent to \emph{fisher-consistency} \citep{liu2007fisher} or \emph{classification-calibration} \citep{steinwart2007how,tewari2007ontheconsistency}; we will also call this property \emph{consistency} and \emph{calibration}.

A limitation of \cref{thm:steinwart} is the lack of elegant bounds for the $\sn{H}$-excess true risk when the Bayes optimal classifier cannot be obtained from a hypothesis in $\sn{H}$.
From a parametric point of view, we want to get true risk bounds that mirror our surrogate risk bounds, that is, bounds on the $\sn{H}$-excess true risk given bounds on the $\sn{H}$-excess surrogate risk.
This means that in a parametric setting we are concerned about using \cref{ass:surrogateRiskBound}, to obtain that for all $\delta' \in (0,1)$, \wpr{} at least $1 - \delta'$,
\[
    \risk(f \circ \hat{h}) - \inf_{h \in \sn{H}}\risk(f \circ h) \leq T_{\sn{H}}(n,\delta').
\]
\Cref{thm:steinwart}, however, implies that for all $\delta' \in (0,1)$ \wpr{} at least $1 - \delta'$, we have
\eq{
    \risk(f \circ \hat{h}) - \inf_{h \in \sn{H}}\risk(f \circ h) \leq \delta^{-1}\pb{ T^{\surr}_{\sn{H}}(n,\delta) + A^\surr_\sn{H} } - A_\sn{H} \label[equation]{eq:parametricBound}
}
where
\[
    A_\sn{H} \eqdef  \inf_{h \in \sn{H}}\risk(f \circ h) -  \inf_{g}\risk(g).
\]
\citet{long2013consistency} have been concerned with guarantees of the above type in specific settings, but we will work with bounds that have the form of \eqref{eq:parametricBound}.
Extending calibration functions to the parametric setting (to the so-called $\sn{H}$-calibration functions) will be left as future work.
Bounds with the form of \eqref{eq:parametricBound} are, however, still informative in the non-parametric setting, when the approximation error can be controlled or is zero.

Next, we discuss calibration functions in more detail and present their forms for some losses in binary \citep[\seea{}][]{bartlett2006convexity,steinwart2007how} and multiclass \citep[\seea{}][]{avilapires2013costsensitive} classification.

\section{The calibration toolset}
\label{sec:calibrationToolset}
In this section we discuss calibration functions in more detail, surveying existing results from the literature.
\Cref{sec:deltaMax} is an instantiation of the theoretical framework of \citet{steinwart2007how} for the classification problem, where a we present the so-called maximum calibration function, which is an important type of calibration function for our analysis.
In \cref{sec:binary} we discuss existing fisher-consistency results and calibration functions for binary classification, and in \cref{sec:multiclass} we have a discussion of the corresponding results for multiclass classification.
\subsection{The maximum calibration function}
\label{sec:deltaMax}

\citet{steinwart2007how} defined a function $\fun{\delta_{\max}}{[0,\infty)}{[0,\infty)}$ that depends on the given surrogate loss and constitutes a key notion for calibration functions.
$\delta_{\max}$ is special because no calibration function for the given surrogate loss is larger than $\delta_{\max}$.
Moreover, if the loss is calibrated, then $\delta_{\max}$ is a calibration function (hence the name \emph{maximum calibration function}).
As a consequence (\see{} \cref{thm:deltaMax}), the surrogate loss is calibrated iff $\delta_{\max}(\eps) > 0$ for all $\eps > 0$, \ie{}, iff $\delta_{\max}(\eps)$ is a calibration function.
Conveniently, any positive lower bound to the maximum calibration function is also a calibration function, which is a useful fact for understanding and calculating calibration functions for specific losses.

In order to define $\delta_{\max}$, we must define three useful concepts (\see{} \cref{def:mset}): The set of scores in $\sn{S}$ whose maximum coordinate is $j$ ($\sn{M}(\sn{S}, j)$), the set of scores that give $\eps$-sub-optimal class predictions ($\sn{T}(\sn{S}, \eps, p)$), and the set of $\eps$-sub-optimal indices with maximum probability ($\sn{J}(\eps, p)$).
A score $s \in \sn{S}$ is \emph{$\eps$-sub-optimal} for a given $p \in \yplex$ if $\max_k p_k - p_{f(s)} \geq \eps$.
On the other hand, a score $s \in \sn{S}$ is \emph{$\eps$-optimal} for a given $p \in \yplex$ if $\max_k p_k - p_{f(s)} < \eps$.
\begin{definition}
\label{def:mset}
Given a set of scores $\sn{S} \subset \rl{\K}$ let, for $\eps \geq 0$ and $p \in \yplex$
\als{
	&\sn{M}(\sn{S}, j) \eqdef \cb{s \in \sn{S} : s_j = \max_{k} s_k}, &
	&\sn{T}(\sn{S}, \eps, p) \eqdef \bigcup_{j : \max_y p_y - p_{j} \geq \eps}\sn{M}(\sn{S},\eps, j), \\
	&\sn{J}(\eps, p) \eqdef \argmax_j \cb{ p_j :  \max_y p_y - p_j  \geq \eps }.
}
\end{definition}
We will override notation and use $\risk^{\surr}_L$ to denote the \emph{pointwise surrogate risk} $\fun{\risk^{\surr}_L}{\sn{S} \times \yplex}{\rl{}}$ for a surrogate loss $\fun{L}{\sn{S} \times \sn{Y}}{\rl{}}$ with $\sn{S} \subset \rl{\K}$ by
\eq{
	\risk^{\surr}_L(s,p) \eqdef \muE{Y \sim p}{ L(s,Y) }, \label[equation]{eq:pointwiseSurrogateRisk}
}
and we will write $\risk^{\surr}$ when the choice of $L$ is clear from context.
The distinction between the surrogate risk and the pointwise surrogate risk can be made by the first argument.
For now, we will be concerned with the pointwise surrogate risk, for which we will define calibration functions.
In \cref{sec:fromPointwise}, we will discuss how to use these calibration functions to obtain calibration functions \emph{per se} (in the sense of \cref{thm:steinwart}), for the surrogate risk.

In \cref{def:deltaMax}, we present two important functions introduced by \citet{steinwart2007how}: $\delta_{\max}(\eps, p)$ and $\delta_{\max}(\eps)$.
The former is the difference between the smallest surrogate risk of any $\eps$-suboptimal score and the optimal surrogate risk.
If any score has surrogate risk closer to the optimal surrogate risk than $\delta_{\max}(\eps, p)$, the score must be $\eps$-optimal \wrt{} $p$.
Confronting this fact with \cref{thm:steinwart}, we see that if $\delta_{\max}$ is positive for all $\eps > 0$, then it is a calibration function.
It is, however, a calibration function only for the pointwise surrogate risk defined in terms of $p \in \yplex$, so in order to define a $\delta_{\max}$ that is a calibration function iff $\delta_{\max}(\eps, p)$ is a calibration function for all $p \in \yplex$, it is natural to take $\delta_{\max}(\eps)$ as the infimum of $\delta_{\max}(\eps, p)$ over all $p$.
If $\delta_{\max}$ is a calibration function, it is called the \emph{maximum calibration function}.
\begin{definition}
\label{def:deltaMax}
Given a set of scores $\sn{S} \subset \rl{\K}$ and a surrogate loss $\fun{L}{\sn{S} \times \sn{Y}}{\rl{}}$, let
\[
	\delta_{\max}(\eps, p) \eqdef \inf_{s \in \sn{T}(\sn{S}, \eps, p) }\risk^{\surr}(s,p) - \inf_{s \in \sn{S}} \risk^{\surr}(s,p)
\]
and
\[
	\delta_{\max}(\eps) \eqdef \inf_{p \in \yplex} \delta_{\max}(\eps, p).
\]
If $\delta_{\max}(\eps) > 0$ for all $\eps > 0$, then it is called the \emph{maximum calibration function}.
\end{definition}
\begin{theorem}[\citealt{steinwart2007how}]
\label{thm:deltaMax}
$\delta_{\max}$ is always non-negative, and no calibration function for the choice of $\sn{S}$ and surrogate loss $L$ (and optionally of $p$) is larger than the corresponding $\delta_{\max}$.
\end{theorem}
As mentioned, \cref{thm:deltaMax} and non-decreasingness of $\delta$ imply that $\delta_{\max}(\eps) = 0$ for some $\eps > 0$ iff the corresponding surrogate loss is not calibrated.

\subsection{From $\delta_{\max}$ to risk bounds}
\label{sec:fromPointwise}

By definition of $\delta_{\max}$, we have that for all $\eps > 0$, $p \in \yplex$ and $s \in \sn{S}$ if
\eq{
	\risk^{\surr}(s, p) - \inf_{s' \in \sn{S}}\risk^{\surr}(s', p) < \delta_{\max}(\eps) \label[equation]{eq:calibrationPreConidtionNoP}
}
then
\eq{
	\risk^{\surr}(s, p) - \inf_{s' \in \sn{S}}\risk^{\surr}(s', p) < \delta_{\max}(\eps, p), \label[equation]{eq:calibrationPreConidtionP}
}
and if \eqref{eq:calibrationPreConidtionP} holds then we also have
\eq{
	\max_{y} p_y - p_{f(s)} < \eps. \label[equation]{eq:calibrationPostCondition}
}
If $\delta(\eps, p) > 0$ for all $\eps > 0$ and all $p \in \yplex$, then \eqref{eq:calibrationPostCondition} holds for all $\eps > 0$, otherwise the guarantee is vacuous.
Moreover, the guarantee breaks down if the infimum in \eqref{eq:calibrationPreConidtionNoP} is unbounded, so we will assume otherwise in \cref{ass:lowerBounded}.
\begin{assumption}
\label{ass:lowerBounded}
Given a surrogate loss $\fun{L}{\sn{S} \times \sn{Y}}{\rl{}}$, we have
\[
	\inf_{s \in \sn{S}, p \in \yplex} \risk^{\surr}(s,p) > -\infty.
\]
\end{assumption}

If the surrogate loss satisfies \cref{ass:lowerBounded}, then any positive function that lower-bounds $\delta_{\max}$ will allow us to obtain a similar guarantee implying \eqref{eq:calibrationPostCondition}.
Therefore, the strategy for calculating calibration functions will be to find a positive function $\delta(\eps,p)$ that lower-bounds $\delta_{\max}(\eps,p)$ for all $p \in \yplex$, or a positive $\delta(\eps) \leq \inf_{p \in \yplex}\delta(\eps,p)$.
Once we have one of these, we can obtain \cref{thm:steinwart} or a similar result by taking the expectation of \eqref{eq:calibrationPreConidtionNoP} or \eqref{eq:calibrationPreConidtionP} with $p = P_{Y|X}$, the \emph{conditional probability} of $Y$ given $X$, defined as $(P_{Y|X})_y \eqdef \cProb{Y = y}{X}$ almost everywhere (\aev{}) for all $y \in \sn{Y}$\footnote{
	For this argument to work, the surrogate risk minimizer must be measurable, which we assume to be the case.
	Formally, following \citet{steinwart2007how}, we assume that for every $\alpha > 0$, there exists a measurable function $\fun{h}{\sn{X}}{\sn{S}}$ \st{} $\risk^{\surr}_L(h(X), P_{Y|X}) - \alpha < \inf_{s \in \sn{S}}\risk^{\surr}_L(s, P_{Y|X})$ \aev{}.
}.
This integration step is used by \citet{zhang2004statistical,chen2006consistency,steinwart2007how} to obtain risk bounds from calibration functions that they define.
We now proceed to results that yield calibration functions for classification, first binary, then multiclass.

\subsection{Calibration functions in binary classification}
\label{sec:binary}

\citet{bartlett2006convexity} characterized $\delta_{\max}$ for margin losses and also characterized the conditions under which $\delta_{\max}$ is a calibration function for convex, lower-bounded $\varphi$\footnote{
	$\varphi$ is \emph{lower-bounded} if $\inf_{t} \varphi(t) > -\infty$.
}.
In \cref{def:deltaBinary}, we introduce $\delta_{\mathrm{binary}}$, which \citet{bartlett2006convexity} have shown to be equal to $\delta_{\max}$ in binary classification (\see{} \cref{thm:bartlett} ahead).
By comparing $\delta_{\max}$ from \cref{def:deltaMax} and $\delta_{\mathrm{binary}}$ from \cref{def:deltaBinary}, and taking \cref{thm:bartlett} into consideration, we can point out a few facts that will shape the conditions that we design to reduce multiclass classification calibration functions, namely, $\delta_{\max}$ to binary classification calibration-like functions, \viz{} $\delta_{\mathrm{binary}}$.
For each $\eps$, the worst-case distribution is $p = \pb{\frac{1 + \eps}{2}, \frac{1 - \eps}{2}}$, \ie{}, $\delta_{\max}(\eps) = \delta_{\max}\pb{\eps, \pb{\frac{1 + \eps}{2}, \frac{1 - \eps}{2}}}$.
Moreover, for every $p \in \yplex$, $\inf_{s \in \sn{T}(\sn{S}, \eps, p) }\risk^{\surr}(s,p) = \risk^{\surr}((0,0),p) = \risk^{\surr\pb{(0,0), \pb{\frac{1}{2}, \frac{1}{2}}}} = \inf_{s \in \sn{T}(\sn{S}, \eps, p) }\risk^{\surr\pb{s, \pb{\frac{1}{2}, \frac{1}{2}}}}$.
In \cref{thm:bartlett} and henceforth, we will denote the \emph{subdifferential} of $\varphi$ at $t \in \rl{}$ by $\partial\varphi(t)$.
\begin{definition}
\label{def:deltaBinary}
Consider a surrogate loss $\fun{L}{\sn{S} \times \sn{Y}}{\rl{}}$ with $\K = 2$ and $\sn{S} \subset \rl{\K}$.
Let
\[
	\delta_{\mathrm{binary}}(\eps) \eqdef \inf_{s\in \sn{S}}\risk^\surr(s, p^{0}) - \inf_{s' \in \sn{S}} \risk^\surr(s', p^{\eps}),
\]
where $p^{\eps} \eqdef \pb{\frac{1 + \eps}{2}, \frac{1 - \eps}{2}}$.
\end{definition}
\begin{theorem}[\citealt{bartlett2006convexity}]
\label{thm:bartlett}
Assume $\K = 2$, $\sn{S} = \sn{S}_0$ and $L(y,s) = \varphi(-s_y)$ where $\varphi$ is convex and lower-bounded.
If $\partial\varphi(0) \subset [0, \infty)$ then
\[
	\delta_{\mathrm{binary}}(\eps) = \varphi(0) + \frac{1}{2}\inf_{t \in \rl{}} \pb{1 + \eps}\varphi(t) + \pb{1 - \eps}\varphi(-t)
\]
and
\[
	\delta_{\max}(\eps) = \delta_{\max}(\eps, p^\eps) = \delta_{\mathrm{binary}}(\eps).
\]
Moreover, $\delta_{\mathrm{binary}}$ is a calibration function iff $\varphi$ has a unique, positive derivative at zero (\ie{}, $\varphi'(0) > 0$).
Finally, if $\varphi$ is convex with a unique, positive derivative at zero, then $\delta_{\mathrm{binary}}$ is convex.
\end{theorem}
\citet{reid2009surrogate} showed an analogue of \cref{thm:bartlett} for \emph{proper losses}, which have domain $\Delta_{2}$ and are defined to satisfy $\risk^{\surr}(p, p) = \inf_{p' \in \yplex^\circ} \risk^{\surr}_{L}(p', p)$ for all $p$ in the interior of $\yplex$, denoted $\yplex^\circ$.

\citet{steinwart2007how} used \cref{thm:bartlett} to calculate $\delta_{\mathrm{binary}}$ for different choices of $\varphi$, which are included in \cref{table:calibrationFunctions}.
\begin{table}[ht]
\centering
\begin{tabular}{ll}
\toprule
Transformation function & $\delta_{\mathrm{binary}}(\eps)$ ($\eps \in (0,1)$) \\
\midrule
Misclassification ($0$-$1$) & $\eps$ \\
Identity & $\nexists$ \\
Linear & $\nexists$ \\
Hinge & $\eps$ \\
Modulus & $\eps$ \\
Squared & $\eps^2$ \\
Truncated square (squared hinge) & $\eps^2$ \\
Exponential & $1 - \sqrt{1 - \eps^2}$ \\
Logistic & $\frac{1}{2}\pb{(1 - \eps)\ln(1 - \eps) + (1 + \eps)\ln(1 + \eps)}$ \\
Sigmoid & $\eps$ \\
Kink ($\tau > 0$) & $\eps$ \\
Kink ($\tau = 0$) & $\nexists$ \\
\bottomrule
\end{tabular}
\caption{Maximum binary classification calibration functions for different transformation functions \citep{steinwart2007how,avilapires2013costsensitive}\label{table:calibrationFunctions}}
\end{table}
The functions $\linear$ and $\eye$ do not have a calibration function because they violate lower-boundedness.
The function $\kink$ with $\tau = 0$ does not have unique derivative at zero, so, by \cref{thm:bartlett}, it is not calibrated.
The calibration function for the margin loss based on $\modulus$ is not reported by \citet{steinwart2007how} but is evident from a result shown by \citet{zou2006margin} and later, independently, by \citet{avilapires2013costsensitive}.
The multiclass version of the result is given later in this text as \cref{lem:lowerBoundedScores}.
Informally, in the binary case \cref{lem:lowerBoundedScores} implies that if $L(y,s) = \varphi(-s_y)$ and $\varphi$ has a minimum $t^{\min}$, then $\delta_{\mathrm{binary}}$ is the same for $L$ with $\sn{S} = \sn{S}_0$ and for $L$ with $\sn{S} = [t^{\min}, -t^{\min}]$.
Thus, we can combine \cref{lem:lowerBoundedScores} with \cref{thm:bartlett} to obtain that the $\delta_{\mathrm{binary}}$ for $L(y,s) = \varphi(-s_y)$ is the same when $\varphi(t) = \ab{1 + t}$ and $\varphi(t) = (1 + t)_+$ (since they are equal in $\sn{S} = [t^{\min}, -t^{\min}] = [-1, 1]$).
Also, $\varphi(t) = (1 + t)^2$ and $\varphi(t) = \pb{(1 + t)_+}^2$ share the same $\delta_{\mathrm{binary}}$.

It would seem that transformation functions leading to superlinear $\delta_{\mathrm{binary}}$, such as the squared transformation function, lead to true risk bounds with slower rates than transformation functions associated with linear $\delta_{\mathrm{binary}}$, such as the hinge loss.
That would be true if the rates of the surrogate risk bounds were asymptotically the same for all the losses, which is not always the case.
As shown by \citet{mammen1999smooth} (see also \citet{boucheron2005theory,bartlett2006convexity}) fast rates can be obtained under the following low-noise condition, known as the \emph{Mammen-Tsybakov noise condition}, which states that there exists $c > 0$ and $\alpha \in [0,1]$ \st{} for every classifier $g$
\eq{
	\Prob{ g(X) \neq g^*(X) } \leq c(R(g) - \inf_{g' \in \sn{G}} R(g'))^\alpha, \label[equation]{eq:MTNC}
}
where $g^*(X) \eqdef \argmin_y \Prob{Y = y|X}$ is the \emph{Bayes-optimal classifier}.
We see that \eqref{eq:MTNC} interpolates between the noiseless case ($\alpha = 1$) and the case where no assumption about the noise is made ($c = 1$, $\alpha = 0$)
Under the Mammen-Tsybakov noise condition, it is possible to get faster rates, as shown by \citet[][Theorems 3 and 5]{bartlett2006convexity}.
Theorem 3 of \citet{bartlett2006convexity}, presented here as \cref{thm:bartlett:MTNC}, improves over \cref{thm:steinwart} by using the Mammen-Tsybakov noise condition.
We can see from \cref{thm:bartlett:MTNC} that the right-hand side of \eqref{thm:bartlett:MTNC:surrRisk} with $\delta(\eps) = \eps^2$ becomes $\frac{1}{c}\eps^{2 - \alpha}$, which gives a fast rate if $\alpha = 1$ and a ``slow'' rate with $\alpha = 0$.
As shown by \citet[Theorem 5]{bartlett2006convexity}, fast rates for the true risk can be obtained by combining \eqref{thm:bartlett:MTNC} and fast rates for surrogate risk, which can be obtained, for example, if scores are bounded in a range $[-t,t]$, and if $\varphi$ is strictly convex and Lipschitz in the interval $[-t,t]$.
\begin{theorem}[Theorem 3, \citealt{bartlett2006convexity}]
\label{thm:bartlett:MTNC}
Assume that there exist $c > 0$ and $\alpha \in [0,1]$ \st{} \eqref{eq:MTNC} holds for every classifier $g$.
Given $\fun{\varphi}{\rl{}}{\rl{}}$ convex, classification-calibrated, for every score function $h$ and $\eps > 0$ the following holds:
If
\eq{
	\risk^{\surr}(h) - \inf_{h'}\risk^{\surr}(h) < c \eps^\alpha \delta_{\mathrm{binary}}\pb{\frac{\eps^{1 - \alpha}}{2c}}, \label[equation]{thm:bartlett:MTNC:surrRisk}
}
then
\[
	\risk(f \circ h) - \inf_{g}\risk(g) < \eps.
\]  
\end{theorem}

\subsection{Calibration functions and multiclass classification}
\label{sec:multiclass}
The panorama of calibration and calibration functions for multiclass classification losses is significantly more disperse that in binary classification, due to the many existing generalizations of the binary margin loss, some of which are collected in \cref{table:scoreLosses}.
\begin{table}
\begin{adjustwidth}{-1in}{-1in}
\centering
 \begin{tabular}{l>{\arraybackslash}l>{\arraybackslash}l}
\toprule
Definition & $\sn{S}$ & Proponent  \\
\midrule
$\sloss{WW}(y,s) \eqdef \sum_{k\neq y} \varphi(s_k  - s_y)$ & $\rl{\K}$  & \citep{weston1998multiclass}\\
$\sloss{CS}(y,s) \eqdef \max_{k \neq y} \varphi(s_k - s_y)$ & $\rl{\K}$ & \citep{crammer2003ultraconservative}  \\
$\sloss{LLW}(y,s) \eqdef \sum_{k \neq y} \varphi(s_k)$ & $\sn{S}_0$  & \citep{lee2004multicategory} \\
$\sloss{Zhang}(y,s) \eqdef \psi(s_y) + F\pb{ \sum_{k=1}^{\K} \varphi(s_k) }$ & $\rl{\K}$  & \citep{zhang2004statistical} \\
$\sloss{RRKA}(y,s) \eqdef \varphi(-s_y) + \sum_{k=1}^{\K} \varphi(s_k)$ & $\rl{\K}$  & \citep{rifkin2004indefense} \\
$\sloss{ZZH}(y,s) \eqdef \varphi(-s_y)$ & $\sn{S}_0$  & \citep{zou2006margin} \\
$\sloss{Liu}(y,s) \eqdef (\K - 2 - s_y)_+$ & $\cb{ s \in \sn{S}_0 : s \geq -\one{\K} }$ & \citep{liu2007fisher} \\
$\sloss{BSKV}(y,s) \eqdef F\pb{\sum_{k \neq y} \varphi(s_k - s_y)}$ & $\rl{\K}$ & \citep{beijbom2014guessaverse} \\
\bottomrule
\end{tabular}
\end{adjustwidth}
\caption{Different score losses. \label{table:scoreLosses}}
\end{table}
For $\sloss{Zhang}$, various choices of $F$ and $\psi$ are possible \citep[\see{}][]{zhang2004statistical}, and we compute calibration functions for some of these in \cref{sec:cases}.
The loss $\sloss{BSKV}$ requires $F$ (strictly) increasing, and corresponds to $\sloss{WW}$ when $F(t) = t$.
The surrogate $\sloss{Zhang}$ also generalizes different entries in \cref{table:scoreLosses}, if we are flexible about $\sn{S}$.
It is easy to see that with $\sn{S} = \sn{S}_0$, $\psi(t) = -\varphi(t)$ and $F(t) = t$, $\sloss{Zhang}$ corresponds to $\sloss{LLW}$; with $\sn{S} = \sn{S}_0$, $\psi(t) = \varphi(-t)$ and $F(t) = 0$ it corresponds to $\sloss{ZZH}$, and with $\sn{S} = \rl{K}$, $\psi(t) = \varphi(-t)$ and $F(t) = 0$ it corresponds to $\sloss{RRKA}$.
It can also be seen that $\sloss{Zhang}$ with $\sn{S} = \rl{\K}$, $\psi(t) = -t$, $F(t) = \ln t$ and $\varphi = \expt$, the logistic regression loss \citep{zhang2004statistical}, is equivalent to $\sloss{Zhang}$ with $\sn{S} = \yplex$, $\psi(t) = -\ln t$ and $F(t) = 0$ (\see{} \cref{prop:logisticSimplex}).

Although $\sloss{WW}$, $\sloss{CS}$, $\sloss{LLW}$, $\sloss{ZZH}$, and $\sloss{Liu}$ all reduce to a margin loss in the binary case, they lead to classifiers with substantially different behaviors in terms of calibration and calibration functions, as we will see next.

\subsubsection{Calibration}
\label{sec:calibration}
We will first provide an overview of calibration results for convex surrogate losses, with calibration functions presented later.
By reduction to the binary case, we get from \cref{thm:bartlett} that $\varphi'(0) > 0$ is a necessary condition for calibration of $\sloss{WW}$, $\sloss{CS}$, $\sloss{LLW}$, $\sloss{ZZH}$, and $\sloss{Liu}$.
Any condition presented ahead for the consistency of these losses will be, of course, in addition to $\varphi'(0) > 0$, and the assumption of \cref{thm:bartlett} that $\varphi$ is convex and lower-bounded.

\citet{lee2004multicategory,zhang2004statistical,liu2007fisher,tewari2007ontheconsistency} showed that $\sloss{LLW}$ is consistent when $\varphi$ is differentiable.
\citet{tewari2007ontheconsistency} provided a counter example of a $\varphi$ with a kink that is not calibrated, similar to \cref{prop:kinkNotCalibrated}.
\begin{restatable}{proposition}{propKinkNotCalibrated}
\label{prop:kinkNotCalibrated}
The loss $\sloss{LLW}$ with $\varphi^{\frac{1}{2}\mathrm{-Kink}}$ is not classification-calibrated.
\end{restatable}
$\sloss{CS}$ is not consistent in general, but it is consistent for distributions $p_i \in \yplex$ there $\max_y p_y \geq \frac{1}{2}$ \citep{zhang2004statistical,liu2007fisher,tewari2007ontheconsistency}.
\citet{zhang2004statistical} defines a property called \emph{order-preservation}, which is sufficient for calibration when $p \in \yplex^\circ$.
\begin{definition}[Order-preservation]
\label{def:orderPreserving}
A loss $\fun{L}{\sn{S} \times \sn{Y}}{\rl{}}$ is \emph{order-preserving} if for all $p \in \yplex^\circ$, $s \mapsto R^\surr(s,p)$ has a minimizer $s^*$ \st{} $p_i > p_j \Rightarrow s^*_i > s^*_j$ for every $i$, $j$.
\end{definition}
$\sloss{ZZH}$ is inconsistent in general \citet{tewari2007ontheconsistency}, but consistent for $p \in \yplex^\circ$ when $\varphi$ allows them to be \emph{order-preserving} \citep{zhang2004statistical}.
In particular, $\sloss{ZZH}$ is order-preserving whenever $\varphi$ is twice-differentiable with $\varphi'(0) > 0$ and $\varphi''(x) > 0$ for all $x \in \rl{}$ \citep{zou2006margin}.

With $\hinge$, the losses $\sloss{ZZH}$ and $\sloss{WW}$ are neither order-preserving nor calibrated \citep{zhang2004statistical,liu2007fisher,tewari2007ontheconsistency}.
Informally, any minimizer $s^*$ of any of these two losses with any convex $\varphi$ with $\varphi'(0) > 0$ satisfies $p_{f(-s^*)} = \min_k p_k$.
In the binary case, this is not a problem because $p_{f(s^*)} = \max_k p_k$.
In the multiclass case, however, with $\sloss{ZZH}$ and $\hinge$, we have $s^*_i = 1$ whenever $p_i > \min_k p_k$, so we can only ``find'' $\min_k p_k$, but not $\max_k p_k$, which is what we are interested in.
\citet{liu2007fisher} modified $\sloss{ZZH}$ with $\hinge$ to obtain the calibrated loss $\sloss{Liu}$.
Interestingly, if we let $\sn{S}' \eqdef \cb{s \in \sn{S}_0 : s \geq -\one{\K}}$, we can show that $\sloss{Liu}$ with $\sn{S} = \sn{S}'$ and $\sloss{ZZH}$ with $\eye$ and $\sn{S} = \sn{S}'$ have the same $\delta_{\max}$ (\see{} \cref{lem:lowerBoundedScores}).

When $\sn{S} = \rl{\K}$ and $L$ is differentiable we can use KKT conditions \citep{boyd2004convex} to get conditions for consistency of some losses, as done by \citet{zhang2004statistical} for many of the losses in \cref{table:scoreLosses}.
For $\sloss{Zhang}$, for example, if the the optimizer $s^*$ exists, it must satisfy, for all $i \in \K$,
\[
	p_i \psi'(s^*_i) + F'\pb{\sum_{k=1}^{\K} \varphi(s^*_k)}\varphi'(s^*_i) = 0.
\]
If, in particular, $F$ is increasing and the function mapping $a$ to the zero of $t \mapsto a\psi'(t) + \varphi'(t)$ is well-defined and increasing for all $a > 0$, we get that $\sloss{Zhang}$ is order-preserving and, thus, calibrated \citep{zhang2004statistical}.
The one-versus-all loss ($\sloss{RRKA}$) is a special case of $\sloss{Zhang}$, so under similar conditions it is calibrated.
\citet{zhang2004statistical} showed that $\sloss{WW}$ with continuously-differentiable and increasing $\varphi$ is also order-preserving and thus consistent \citep[see also ][for an alternative proof]{tewari2007ontheconsistency}.
We can also see that $\sloss{BSKV}$ is order-preserving and thus calibrated under the same conditions on $\varphi$ as $\sloss{WW}$, since $F$ is (strictly) increasing (\cf{} \citet{zhang2004statistical} and \citet{beijbom2014guessaverse}).

\subsubsection{Calibration functions}
\label{sec:calibrationFunctions}

Calibration functions have been calculated for $\sloss{RRKA}$ \citep{zhang2004statistical} and for $\sloss{LLW}$ \citep{chen2006consistency,avilapires2013costsensitive}, with specific choices of $\varphi$.
The first result we present is \cref{thm:zhang}, due to \citet{zhang2004statistical}.
The function $V(p)$ defined in \cref{thm:zhang} is the optimal (binary) surrogate risk when $Y \sim \Bern(p)$ and $\sn{S} = \sn{S}_0$, and the condition in \eqref{eq:thm:zhang:condition} corresponds to strong concavity \citep[][Definition 2.1.2 and Theorem 2.1.9]{nesterov2013introductory}.
In particular, (as pointed out by \citet{zhang2004statistical}) if $V''(p) \leq -c'$ for all $p \in [0,1]$ and some $c' > 0$, then \eqref{eq:thm:zhang:condition} is satisfied with $c = \frac{\sqrt{c'}}{2}$.
We will recover \cref{thm:zhang} as a special case of our \cref{lem:ZhangCondPairing}, at which the argument used to prove \cref{thm:zhang} will be evident.
\begin{theorem}[\citealt{zhang2004statistical}]
\label{thm:zhang}
Consider $\sloss{RRKA}$ with $\sn{S} = \rl{\K}$ and $\varphi$ convex, lower-bounded, differentiable \st{} $\varphi(t) > \varphi(-t)$ for all $x > 0$.
Given $p \in [0,1]$, let
\[
	V(p) \eqdef \inf_{t \in \rl{}} p\varphi(-t) + (1 - p)\varphi(t).
\]
$V$ is concave, and if there is a $c > 0$ \st{} for all $p, p' \in [0,1]$
\eq{
	\frac{1}{2}V(p) + \frac{1}{2}V(p') \leq V\pb{\frac{p + p'}{2}} - c^2(p - p')^2, \label[equation]{eq:thm:zhang:condition}
}
then, for all $\eps > 0$ and all $p \in \yplex$ \st{} $\max_k p_k - \min_k p_k \geq \eps$, we also have
\[
	\delta_{\max}(\eps,p) \geq c \eps^2
\]
and $\delta_{\max}(\eps) \geq c \eps^2$.
\end{theorem}

\citet{chen2006consistency} derive calibration functions for $\sloss{LLW}$ with $\varphi$ convex, differentiable, increasing, and satisfying $\lim_{t \ra -\infty} \varphi(t) = 0$ and $\lim_{t \ra \infty} \varphi(t) = \infty$, based on an assumption that can be shown to imply the existence of a calibration function.
%
%
%
\citet{avilapires2013costsensitive} proved a result for more general $\varphi$ in a cost-sensitive setting, which we present as \cref{thm:avilapires}.
\Cref{thm:avilapires} has a form closer to \cref{thm:bartlett}, with two notable differences: The calibration function depends on both $\eps$ and $p$, rather than just $\eps$, and the form for $\delta(\eps,p)$ takes an infimum over $\theta \geq 0$.
We can use \cref{thm:avilapires} to obtain a calibration function in terms of $\eps$ alone, but we will take a slightly different route and arrive at such a result in \cref{sec:cases:S0}.
If the infimum over $\theta \geq 0$ in \eqref{eq:deltaMaxTheta} is taken at $0$, then \cref{thm:avilapires} becomes a generalization of \cref{thm:bartlett} (when $\sloss{LLW}$ is taken as the multiclass generalization of the margin loss).
The infimum is taken $\theta = 0$ for $\hinge$, $\expt$, and $\squared$ (\see{} \cref{sec:cases:S0}).
\begin{theorem}[\citealt{avilapires2013costsensitive}]
\label{thm:avilapires}
Consider $\sloss{LLW}$ with $\varphi$ convex and let
\begin{equation}
	\delta(\eps,p) \eqdef \inf_{\theta \geq 0}  \cb{ (2 - p_{j_0} - p_{\jeps} )\varphi(\theta) - \inf_{s \in \rl{}} \cb{ (1-p_{j_0}) \varphi(\theta + s) + (1- p_{\jeps} )\varphi(\theta-s) }}, \label[equation]{eq:deltaMaxTheta}
\end{equation}
where $j_0 \in \sn{J}(0, p)$ and $\jeps \in \sn{J}(\eps,p)$.
If $\delta(p,\eps) > 0$ for all $\eps$, then it is a calibration function for distribution $p$.
\end{theorem}

The proof of \cref{thm:avilapires} entails a series of steps to lower-bound $\delta_{\max}(\eps)$ with $\sloss{LLW}$ as the surrogate loss.
\Cref{thm:zhang} can be seen to also entail this ``reduction'' of multiclass classification calibration functions to binary classification calibration functions.
In the next section, we will generalize and reuse the analysis of \citet{avilapires2013costsensitive} in order do most of the work involved in obtaining calibration functions for other surrogate losses.

To conclude this section, we discuss fast rates in multiclass classification.
We are able to generalize {thm:bartlett:MTNC} to the multiclass case with minor effort, which gives us \cref{thm:MTNC}. 
In \cref{thm:MTNC}, the calibration function $\delta$ used to obtain the bound must be convex and non-decreasing.
Non-decreasingness is not a real restriction, since if $\delta$ is a calibration function then $\eps \mapsto \sup_{\eps' \leq \eps}\delta(\eps)$ is a non-decreasing calibration function that is not worse than (that is, greater than or equal to) $\delta$.
Convexity holds for $\delta_{\mathrm{binary}}$, which will allow our calibration functions in this paper to satisfy the assumption of \cref{thm:MTNC}, since, as we will see, in our results we will lower-bound $\delta_{\max}(\eps) \geq \delta_{\mathrm{binary}}(\eps)$ for different losses.
\begin{theorem}
\label{thm:MTNC}
	Assume that there exist $c > 0$ and $\alpha \in [0,1]$ \st{} \eqref{eq:MTNC} holds for every classifier $g$.
	Given a surrogate loss $\fun{L}{\sn{S} \times \sn{Y}}{\rl{}}$, assume that it has calibration function $\fun{\delta}{(0,\infty)}{(0,\infty)}$ and that $\delta$ is convex and non-decreasing.
	Then for every score function $h$ and $\eps > 0$ the following holds:
	If
	\[
		\risk^{\surr}(h) - \inf_{h'}\risk^{\surr}(h) < c \eps^\alpha \delta\pb{\frac{\eps^{1 - \alpha}}{2c}},
	\]
	then
	\[
		\risk(f \circ h) - \inf_{g}\risk(g) < \eps.
	\]  
\end{theorem}

\Citet[Theorem 2]{mroueh2012multiclass} showed a fast-rate result for a simplex-coding least-squares-like loss proposed by them (called S-LS).
It can be shown that Theorem 2 of \citet{mroueh2012multiclass} and \cref{thm:MTNC} are equivalent up to constant factors.
By slightly modifying the proof of Lemma 5 of \citet{bartlett2006convexity}, we can also show that \eqref{eq:MTNC} and Definition 2 of \citet{mroueh2012multiclass} applied to S-LS are also equivalent, with $q = \frac{\alpha}{1 - \alpha}$.

\begin{proof}[Proof of \cref{thm:MTNC}]
This proof is an adaptation of the proof of \cref{thm:bartlett:MTNC} \citep[Theorem 3,][]{bartlett2006convexity} for the multiclass case.
By simple algebra, we can see that the following holds almost surely for every classifier $g$ (where $g^*$ is the Bayes-optimal classifier):
\als{
	&\cProb{g(X) \neq Y}{X} - \cProb{g^*(X) \neq Y}{X}\\
	&~~= \cProb{Y = g^*(X)}{X} - \cProb{Y = g(X)}{X} \\
	&~~= \ind{g(X) \neq g^*(X)}\pb{\cProb{Y = g^*(X)}{X} - \cProb{Y = g(X)}{X}}.
}
Therefore, by the Mammen-Tsybakov noise condition,
\als{
	\risk(g) - \inf_{g'}\risk(g')
	&~~= \E{\ind{g(X) \neq g^*(X)}\pb{\cProb{Y = g^*(X)}{X} - \cProb{Y = g(X)}{X}}} \\
	&~~\leq \Prob{g(X) \neq g^*(X)} \\
	&~~\leq c(\risk(g) - \inf_{g'}\risk(g'))^\alpha.
}

Fix any score function $h$.
By contrapositive of the calibration guarantee, we have that for all $\eps > 0$
\[
	\cProb{Y = g^*(X)}{X} - \cProb{Y = f(h(X))}{X} \geq \eps
\]
implies that
\[
	\cE{L(h(X), Y)}{X} - \inf_{s \in \sn{S}}\cE{L(s, Y)}{X} \geq \delta(\eps).
\]
Hence,
\[
	\cE{L(h(X), Y)}{X} - \inf_{s \in \sn{S}}\cE{L(s, Y)}{X} \geq 
	\delta\pb{\cProb{Y = g^*(X)}{X} - \cProb{Y = f(h(X))}{X}}.
\]

We will use the shorthand
\[
	D(h) \eqdef \cProb{g(X) \neq Y}{X} - \cProb{g^*(X) \neq Y}{X}.
\]
Therefore, for any $t \geq 0$,
\als{
	&\risk(f \circ h) - \inf_{g'}\risk(g') \\
	&~~=\E{ D(h) \ind{D(h) \leq t} } + \E{ D(h) \ind{D(h) > t} } \\
	&~~\leq ct(\risk(f \circ h) - \inf_{g'}\risk(g'))^\alpha + \E{ D(h) \ind{D(h) > t} } \\
	&~~\leq ct(\risk(f \circ h) - \inf_{g'}\risk(g'))^\alpha + \frac{t}{\delta(t)}\E{ \delta(D(h)) } \\
	&~~\leq ct(\risk(f \circ h) - \inf_{g'}\risk(g'))^\alpha + \frac{t}{\delta(t)}(\risk^{\surr}(h) - \inf_{h'}\risk^{\surr}(h))
}
To obtain the third line, we have Item 2 in Lemma 2 of \citet{bartlett2006convexity}: If $u$ is a convex function with $u(0) = 0$, then for all $a > 0$ and $b \in [0,a]$ we have $u(a) \leq \frac{a}{b}u(a)$.
Now, we have assumed that $\delta$ is convex, but $\delta$ has domain $(0, \infty)$.
We can extend $\delta$ to $[0,\infty)$ by defining $\delta(\eps) = \limsup_{\eps \to 0}\delta(\eps)$.
But
\[
	\limsup_{\eps \to 0}\delta(\eps) \leq \limsup_{\eps \to 0}\delta_{\max}(\eps) = 0,
\]
so
\[
	D(h) \ind{D(h) > t}	\leq \frac{t}{\delta(t)}\delta(D(h))
\]
almost surely.

By taking $t = \frac{1}{2c}(\risk(f \circ h) - \inf_{g'}\risk(g'))^{1 - \alpha}$ and performing some manipulations, we get that
\[
	\risk^{\surr}(h) - \inf_{h'}\risk^{\surr}(h) \geq c(\risk(f \circ h) - \inf_{g'}\risk(g'))^\alpha \delta\pb{ \frac{(\risk(f \circ h) - \inf_{g'}\risk(g'))^{1 - \alpha}}{2c} }.
\]
Since $\delta$ is non-decreasing, we get that for every $\eps > 0$
\[
	\risk(f \circ h) - \inf_{g'}\risk(g') \geq \eps
\]
implies that
\[
	\risk^{\surr}(h) - \inf_{h'}\risk^{\surr}(h) \geq c\eps^\alpha \delta\pb{ \frac{\eps^{1 - \alpha}}{2c} },
\]
and the result follows by contrapositive.
\end{proof}

\section{Streamlining the derivation of calibration functions for multiclass classification}
\label{sec:streamlining}
\Cref{thm:avilapires} lower-bounds the $\delta_{\max}(\eps,p)$ of $\sloss{LLW}$ by a quantity that resembles the maximum calibration function of a binary loss ($\delta_{\mathrm{binary}}$).
Obtaining such lower-bounds for the multiclass $\delta_{\max}$ can be a good idea because $\delta_{\mathrm{binary}}$ is often easier to compute than $\delta_{\max}$, once the surrogate loss has been chosen.
Moreover, since $\delta_{\mathrm{binary}}$ is convex, we can apply \cref{thm:MTNC}.
\Cref{thm:zhang} lower-bounds $\delta_{\max}(\eps)$ based on a condition that is ``easy to verify'', \ie{}, it requires some calculations based on a binary classification loss.
We deem it acceptable for our results to also impose certain conditions that can be computed from a binary classification loss.

Throughout this section, we will consider a score loss $\fun{L}{\sn{S} \times \sn{Y}}{\rl{}}$ and the appropriate score set $\sn{S} \subset \rl{\K}$.
Most statements will omit quantifiers on $\K$, in which case the statements apply for any $\K \geq 2$.
In order to carry out our analysis in general terms, we need to express the surrogate loss in a general form that we will take advantage of.
We say the function $\fun{\adj}{\sn{S}}{\rl{}}$ is an \emph{adjustment} function for $L$ if it satisfies $\muE{Y \sim p}{L(s,Y)} = \muE{Y \sim p}{L(s,Y) - \adj(s)} + \adj(s)$ for all $s \in \sn{S}$ and $p \in \yplex$.
For example, $\sloss{Zhang}$ has a natural adjustment: $\adj(s) = F(\sum_{k=1}^{\K} \varphi(s_k))$.
Adjustment functions will be useful for reusing results; for example, $\sloss{LLW}$ can be written as $\sloss{Zhang}$ with $\psi(x) = -\varphi(x)$ and $F(x)$ (but $\sn{S} = \sn{S}_0$), so that some results that apply to $\sloss{Zhang}$ may also apply to $\sloss{LLW}$.
Given the surrogate loss $L$ and an adjustment $\adj$, we define for the \emph{pseudo-risk} for $s \in \sn{S}$ and $p \in \rl{\K}$:
\[
	\risk'(s,p) \eqdef \adj(s) + \sum_{k=1}^{\K} p_k (L(s,k) - \adj(s)).
\]
Note that if $p \in \yplex$, then the pseudo-risk is the pointwise surrogate risk, and also note that we do not multiply the first term (the adjustment) by $\sum_{k=1}^{\K} p_k$.
Moreover, the dependence of $L$ on $\K$ is implicit, so, \eg{}, if we write $\risk'((s_1,s_2),(p_1,p_2))$ we refer to $L$ with $\K = 2$.

The goal of this section is to present a set of conditions that allow us to reuse the reduction analysis of \citet{avilapires2013costsensitive} to generalize and improve \cref{thm:avilapires} for different surrogate losses.
Once these conditions are verified, we immediately obtain \cref{lem:deltaMaxFramework}, which is similar to \cref{thm:bartlett}, but applies to multiclass losses.
We first present \cref{lem:deltaMaxFramework}, then we introduce \cref{cond:infAtJepsJ0,cond:pairing,cond:zetaOfEps,cond:infWithP}, which are essentially the steps used to prove \cref{lem:deltaMaxFramework}.
It is important to point out that \cref{lem:deltaMaxFramework} does not necessarily imply that $\delta_{\max} = \delta_{\mathrm{binary}}$ when $\K > 2$, since $\delta_{\max}(\eps, p^{\eps})$ is not necessarily equal to $\delta_{\max}(\eps, p')$ when $(p_1,p_2) = p_{\eps}$ and $p_k = 0$ for $k > 2$.
After presenting \cref{lem:deltaMaxFramework}, we introduce \cref{cond:infAtJepsJ0,cond:pairing,cond:zetaOfEps,cond:infWithP}, and illustrate how they are used to prove \cref{lem:deltaMaxFramework}.
The proof itself can be found in \cref{app:proofs}.
\begin{restatable}{lemma}{lemDeltaMaxFramework}
\label{lem:deltaMaxFramework}
Consider the surrogate loss $L$ with score set $\sn{S}$. 
If $L$ satisfies \cref{cond:infAtJepsJ0,cond:pairing,cond:zetaOfEps} holds, then for all $\eps > 0$ we have
\[
	\delta_{\max}(\eps) \geq \zeta(\eps).
\]
Moreover, if \cref{cond:infWithP} holds, then
\[
	\delta_{\max}(\eps) \geq \delta_{\mathrm{binary}}(\eps).
\]
\end{restatable}

\Cref{lem:deltaMaxFramework} essentially reduces $\delta_{\max}$ to a binary calibration function, $\delta_{\mathrm{binary}}$.
We will illustrate how \cref{cond:infAtJepsJ0,cond:pairing,cond:zetaOfEps,cond:infWithP} are employed in the proof of \cref{lem:deltaMaxFramework} with $L = \sloss{e.g.} \eqdef \sloss{Zhang}$, $\sn{S} = \rl{}$, $\psi(t) = -t$, $F(t) = t$, $\varphi = \expt$ (a loss akin to logistic regression, where $F(t) = \ln t$ instead, \see{} \cref{sec:cases:logistic}).
As indicated earlier, $\adj(s) = F(\sum_{k=1}^{\K} \varphi(s_k))$ for $\sloss{Zhang}$.
For convenience, in our illustration we will assume that $p \in \yplex^\circ$, in which case
\[
	\inf_{s \in \sn{T}(\sn{S}, \eps, p)} \risk^{\surr}(s,p)
\]
has a minimizer for all $\varepsilon \geq 0$.

\Cref{cond:infAtJepsJ0}, states that a surrogate risk minimizer $s^*$ of the $\varepsilon$-sub-optimal scores will satisfy $\max_k s^*_k =  s^*_{j_0} = s^*_{\jeps}$.
As in \cref{thm:avilapires}, we take $j_0 \in \sn{J}(0, p)$ and $\jeps \in \sn{J}(\eps,p)$ (and the choice is arbitrary whenever multiple choices exist).
Intuitively, this makes sense for our example loss $\sloss{e.g.}$: In
\eq{
	\inf_{s \in \sn{T}(\sn{S}, \eps, p)} \sum_{k=1}^{\K} -p_k s_k + \adj(s) \label[equation]{eq:exampleObjective:infAtJepsJ0}
}
if we have $s_{j_0} < s_{\jeps}$, we can always decrease the objective above by decreasing $s_{\jeps}$ and increasing $s_{j_0}$ (since $(1 - p_{j_0}) \leq (1 - p_{\jeps})$).
Since the minimizer $s^*$ must satisfy $s^*_{j_0} \leq s^*_{\jeps} = \max_{k} s^*_k$, we must have $s^*_{\jeps} = s^*_{j_0}$.
\begin{condition}
\label{cond:infAtJepsJ0}
For the surrogate loss $L$ with score set $\sn{S}$, for all $\eps \geq 0$ and $p \in \yplex$,
\[
	\inf_{s \in \sn{T}(\sn{S}, \eps, p)} \risk^{\surr}(s,p) = \inf_{s \in \sn{M}(\sn{S}, \jeps) \cap \sn{M}(\sn{S}, j_0)} \risk^{\surr}(s,p).
\]
where $j_0 \in \sn{J}(0, p)$ and $\jeps \in \sn{J}(\eps,p)$.
\end{condition}
If \cref{cond:infAtJepsJ0} is satisfied, then \eqref{eq:exampleObjective:infAtJepsJ0} becomes equal to
\eq{
	\inf_{s \in \sn{T}(\sn{S}, \eps, p')} \sum_{k=1}^{\K} -p'_k s_k + \adj(s) \label[equation]{eq:exampleObjective:pPrime}
}
where $p'_k \eqdef p_k$ for $k \notin \cb{\jeps, j_0}$, $p'_{\jeps} = p'_{j_0} \eqdef \frac{p_{\jeps} + p_{j_0}}{2}$.

\Cref{cond:pairing} effectively reduces the multiclass calibration function to a binary $\delta_{\mathrm{binary}}$-like function, given in terms of pseudo-risks.
With $s^*$ as the minimizer of \eqref{eq:exampleObjective:pPrime}, we can lower-bound
\eq{
	\sup_{s \in \sn{S}} -\risk^{\surr}(s, p) \geq \sup_{\substack{s \in \sn{S}: \\ s_k = s^*_k : k \notin \cb{\jeps, j_0}}} -\risk^{\surr}(s, p) = \sup_{\substack{s \in \sn{S}: \\ s_k = s^*_k : k \notin \cb{\jeps, j_0}}} -\risk^{\surr}(s, p'), \label[equation]{eq:exampleObjective:fixingCoordinates}
}
which, combined with \eqref{eq:exampleObjective:pPrime}, gives us \cref{cond:pairing} for our example.
\begin{condition}
\label{cond:pairing}
Consider the surrogate loss $L$ with score set $\sn{S}$ and adjustment $\adj$.
For all $i,j \in \sn{Y}$ \st{} $i \neq j$ and all $p \in \yplex$,
\als{
	&\inf_{s \in \sn{M}(\sn{S}, i) \cap \sn{M}(\sn{S}, j) }\sup_{\substack{s' \in \sn{S}:\\ s'_{k} = s_{k}, k \neq i,j}} \risk^{\surr}(s,p) - \risk^{\surr}(s',p) \\
	&~~\geq \inf_{s \in \sn{S}}\sup_{s' \in \sn{S}}\risk'(s,  (\bar{p}, \bar{p})) - \risk'(s', (p_i, p_j)),
}
where $\bar{p} = \frac{p_i + p_j}{2}$.
\end{condition}
\Cref{cond:zetaOfEps} is then used to lower-bound the $\delta_{\mathrm{binary}}$-like function of \cref{cond:pairing} by some non-decreasing function $\zeta$ of the gap between $p_{j_0}$ and $p_{\jeps}$, so that we can lower-bound $\zeta(p_{j_0} - p_{\jeps}) \geq \zeta(\eps)$, \ie{}, $\zeta$ will be our calibration function.
\Cref{cond:infWithP}, presented in the sequence, states \cref{cond:zetaOfEps} holds with $\zeta = \delta_{\mathrm{binary}}$.
It is possible to show that \cref{cond:infWithP} does hold for $\sloss{e.g.}$ (\see{} \cref{lem:ZhangCondInfWithP}).
\begin{condition}
\label{cond:zetaOfEps}
Consider the surrogate loss $L$ with score set $\sn{S}$ and adjustment $\adj$.
There exists a non-decreasing $\fun{\zeta}{[0,1]}{\rl{}}$ \st{} for all $\eps > 0$, and $p_1,p_2$ \st{} $0 \leq p_2 \leq p_1 - \eps$ and $p_1 + p_2 \leq 1$,
\als{
	&\inf_{s \in \sn{S}} \risk'(s, (\bar{p},\bar{p})) - \inf_{s' \in \sn{S}} \risk'(s', (p_1, p_2)) \\
	&~~\geq \zeta(p_1 - p_2).
}
where $\bar{p} = \frac{p_1 + p_2}{2}$.
\end{condition}
\begin{condition}
\label{cond:infWithP}
\Cref{cond:zetaOfEps} holds with $\zeta = \delta_{\mathrm{binary}}$.
\end{condition}

Using the steps exemplified above for $\sloss{e.g.}$, we can use \cref{cond:infAtJepsJ0,cond:pairing,cond:zetaOfEps,cond:infWithP} to obtain \cref{lem:deltaMaxFramework}.
At this point we are almost ready to verify \cref{cond:infAtJepsJ0,cond:pairing,cond:zetaOfEps,cond:infWithP} for specific cases, but before we do so we ought to introduce some useful results and conditions.
When illustrating \cref{cond:infAtJepsJ0}, we argued that we could decrease the objective in \eqref{eq:exampleObjective:infAtJepsJ0} by increasing $s_{j_0}$ and decreasing $s_{\jeps}$, provided that $s_{j_0} < s_{\jeps}$.
In general, this may not be doable, \eg{}, if classes are treated ``differently'' by $L$.
``Equal'' treatment is, in precise terms, symmetry in the sense of \cref{cond:symmetry}.
\begin{condition}
\label{cond:symmetry}
Given the surrogate loss $L$ with score set $\sn{S}$, for any permutation matrix $P$, any $s \in \sn{S}$ and any $p \in \yplex$, we have $Ps \in \sn{S}$ and $\risk^{\surr}(s,p) = \risk^{\surr}(Ps, Pp)$.
\end{condition}
The following result is a useful observation: If $\K = 2$ and $L$ satisfies \cref{cond:symmetry}, then for any $p \geq 0$ the minimizer $s^* \in \rl{2}$ of the surrogate pseudo-risk \wrt{} $(p, p)$ satisfies $s^*_1 = s^*_2$.
The second statement in \cref{prop:subOptimalSurrogateRisk} is a straightforward consequence of the first result and \cref{cond:infAtJepsJ0}.
\begin{restatable}{proposition}{propSubOptimalSurrogateRisk}
\label{prop:subOptimalSurrogateRisk}
Consider the surrogate surrogate loss $L$ with score set $\sn{S} \subset \rl{2}$ and adjustment $\adj$.
If $L$ satisfies \cref{ass:lowerBounded,cond:symmetry}, then for any $p \in \rl{2}_+$
\[
	\inf_{s \in \sn{S}} \risk'\pb{s, (\bar{p}, \bar{p})} = \inf_{\substack{s \in \sn{S}: \\  s_1 = s_2}} \risk'\pb{s, (p_1, p_2)},
\]
where $\bar{p} = \frac{p_1 + p_2}{2}$.
Moreover, \cref{cond:infAtJepsJ0} implies that, for any $\eps > 0$, if $p \in \simplex{2}$ and $\sn{T}(\sn{S}, \eps, p) \neq \emptyset$, then
\[
	\inf_{s \in \sn{S}} \risk^{\surr}\pb{s, \pb{\frac{1}{2}, \frac{1}{2}}} = \inf_{s \in \sn{T}(\sn{S}, \eps, p)} \risk^{\surr}(s,p).
\]
\end{restatable}
When combined with the \cref{cond:infWithP}, \cref{prop:subOptimalSurrogateRisk} gives us a tight result for $\K = 2$, \cref{lem:binaryTightness}, which recovers \cref{thm:bartlett}, and shows that $\delta_{\mathrm{binary}}$ is the largest calibration function, \ie{}, $\delta_{\max}$.
We present this equality result as \cref{lem:binaryTightness}, a result originally shown by \citet{bartlett2006convexity} for the margin loss.
As a consequence we can recover another result by \citet{bartlett2006convexity} (for the margin loss, with $\K = 2$): $L$ is calibrated iff $\delta_{\mathrm{binary}}(\eps) > 0$ for all $\varepsilon > 0$.
\begin{restatable}{lemma}{lemmaBinaryTightness}
\label{lem:binaryTightness}
Consider the surrogate loss $L$ with score set $\sn{S}$, and assume \cref{cond:infAtJepsJ0,cond:infWithP,cond:symmetry} are satisfied.
For $\K = 2$ and all $\eps > 0$, we have
\[
	\delta_{\max}(\eps) = \delta_{\mathrm{binary}}(\eps).
\]
\end{restatable}
Proving \cref{cond:pairing} requires us to take advantage of structure in $L$ and account for the choice of $\sn{S}$.
When $\sn{S} = \sn{S}_0$, it will be convenient to break \cref{cond:pairing} down into two additional conditions from which it follows: \Cref{cond:pairingSumToZero,cond:sumToZeroFreeLunch}.
Having the sum-to-zero constraint forces us to introduce $\inf_{\theta \geq 0}$ in \cref{cond:pairingSumToZero}.
For example, with $L = \sloss{LLW}$ proceeding similarly to \eqref{eq:exampleObjective:fixingCoordinates}, it is not hard to see that the result in \cref{cond:pairingSumToZero} holds with equality.
\Cref{cond:sumToZeroFreeLunch}, in turn, allows us to eliminate the infimum.
In \cref{sec:cases:S0} we will discuss to what extent we are able to verify \cref{cond:sumToZeroFreeLunch} for the loss $\sloss{LLW}$ with different choices of $\varphi$.
\begin{condition}
\label{cond:pairingSumToZero}
Consider the surrogate loss $L$ with score set $\sn{S} = \sn{S}_{0}$ and adjustment $\adj$.
For all $i,j \in \sn{Y}$ \st{} $i \neq j$ and all $p \in \yplex$,
\als{
	&\inf_{\substack{s \in \sn{M}(\sn{S}, i) \cap \sn{M}(\sn{S}, j)} }\sup_{\substack{s' \in \sn{S}:\\ s'_{k} = s_{k}, k \neq i,j}} \risk^{\surr}(s,p) - \risk^{\surr}(s',p) \\
	&~~\geq \inf_{s \geq 0}\sup_{s' \in \rl{}}\pb{\risk'((s, s),  (\bar{p}, \bar{p})) - \risk'((s + s', s - s'), (p_i, p_j))},
}
where $\bar{p} \eqdef \frac{p_i + p_j}{2}$.
\end{condition}
\begin{condition}
\label{cond:sumToZeroFreeLunch}
Consider the surrogate loss $L$ with score set $\sn{S} = \sn{S}_0$ and adjustment $\adj$.
For all $p_1, p_2 \geq 0$
\als{
	&\inf_{s \geq 0}\sup_{s' \in \rl{}}\pb{\risk'((s, s),  (\bar{p}, \bar{p})) - \risk'((s + s', s - s'), (p_1, p_2))} \\
	&~~= \inf_{s \in \sn{S}} \risk'(s, (\bar{p},\bar{p})) - \inf_{s' \in \sn{S}} \risk'(s', (p_1, p_2))
}
where $\bar{p} \eqdef \frac{p_1 + p_2}{2}$.
\end{condition}

To conclude, we present two assumptions that will allow us to reuse results between similar losses with different score sets, in particular to carry some results from $\sloss{Zhang}$ to $\sloss{LLW}$.
\Cref{ass:swapping} states that we can swap coordinates of scores and obtain a valid score (one belonging to $\sn{S}$), while \cref{ass:averaging} is a weaker version of convexity of $\sn{S}$.
\Cref{ass:averaging} will allow us to use convexity of $L$ and Jensen's inequality to show \cref{cond:infAtJepsJ0} (\see{} \cref{lem:ZhangCondInfAtJepsJ0}).
We point out that $\rl{\K}$, $\sn{S}_0$ and $\yplex$ all satisfy \cref{ass:swapping,ass:averaging}.
\begin{assumption}
\label{ass:swapping}
For any $i,j$, and any $s \in \sn{S}$, if $s \in \sn{M}(\sn{S}, j)$ then $s' \in \sn{M}(\sn{S}, i)$, where $s'$ is defined by
\als{
	s'_k \eqdef \begin{cases}
		s_{j} & k = i, \\
		s_{i} & k = j, \\
		s_k & k \neq i,j.
	\end{cases}
}
\end{assumption}
\begin{assumption}
\label{ass:averaging}
For any $i,j$, and any $s \in \sn{M}(\sn{S}, j)$ \st{} $s_j > s_i \geq s_k$ for all $k \neq j$, we have $s' \in \sn{M}(\sn{S}, j) \cap \sn{M}(\sn{S}, i)$, where $s'$ is defined by
\als{
	s'_k \eqdef \begin{cases}
		\frac{s_j + s_i}{2} & k = i \mbox{ or } k = j, \\
		s_k & k \neq i,j.
	\end{cases}
}
\end{assumption}

\section{Case studies}
\label{sec:cases}

In this section we instantiate the reduction analysis in order to obtain calibration functions for specific losses.
We aim to illustrate the applicability of the streamlined analysis developed in \cref{sec:streamlining}.
We apply our analysis to $\sloss{Zhang}$, to the multiclass logistic regression loss, and to $\sloss{LLW}$, respectively, in \cref{sec:cases:Zhang,sec:cases:logistic,sec:cases:S0}.

\subsection{The background discrimination loss}
\label{sec:cases:Zhang}

The background discrimination loss $\sloss{Zhang}$ \citep{zhang2004statistical} admits two formulations: \emph{decoupled}, with $F(t) = t$, and \emph{coupled}.
We verify \cref{cond:pairing} only for the decoupled formulations, while the others are verified regardless of $F$ (but with restrictions on $\psi$), and we summarize results for the specific decoupled variants discussed by \citet{zhang2004statistical} in \cref{table:decoupledLZhang}.
As for the coupled variants, \citet{zhang2004statistical} only discusses two, the logistic regression loss and the loss with $\psi(t) = -t$, $\varphi(t) = \ab{t}^{a'}$ and $F(t) = \frac{1}{a}\ab{t}^{\frac{a}{a'}}$, for $a,a' > 1$.
We instantiate the reduction analysis for the former variant (\cref{sec:cases:logistic}), but the analysis for the latter is left as future work.
In this section we also recover \cref{thm:zhang}, since $\sloss{RRKA}$ is a special case of $\sloss{Zhang}$.

First, we will verify \cref{cond:infAtJepsJ0} in three steps: \Cref{prop:ZhangSymmetric,lem:ZhangInfAtJeps,lem:ZhangCondInfAtJepsJ0}.
Having \cref{cond:symmetry} (verified by \cref{prop:ZhangSymmetric}) will allow us to use the following argument to verify \cref{cond:infAtJepsJ0}: If a score $s \in \sn{T}(\sn{S}, \eps, p)$ satisfies $\max_k s_k > s_{\jeps}$, we can swap $\argmax_k s_k$ and $s_{j_\eps}$ to obtain a score vector with the same or lower pointwise surrogate risk, and which also belongs to $\sn{T}(\sn{S}, \eps, p)$.
This is the result of \cref{lem:ZhangInfAtJeps}, and in order to guarantee that the swap will give us an element in $\sn{S}$, we impose \cref{ass:swapping}.
We proceed by arguing that if a score $s \in \sn{M}(\sn{S}, \jeps)$ does not satisfy $s_{j_0} = \max_{k \neq \jeps} s_k$ then we can swap $\argmax_{k \neq \jeps} s_k$ and $s_{j_0}$ to obtain a score vector with equal or lower pointwise surrogate risk, and which also belongs to $s \in \sn{M}(\sn{S}, \jeps)$.
Thanks to \cref{ass:averaging}, we can then use Jensen's inequality to show that, with $s \in \sn{M}(\sn{S}, \jeps)$ satisfying $s_{j_0} = \max_{k \neq \jeps} s_k$, the surrogate risk does not increase if we average the $\jeps$ and $j_0$-th coordinates, \ie{}, take $s'$ \st{} $s'_{\jeps} = s'_{j_0} = \frac{s_{\jeps} + s_{j_0}}{2}$ and equal for the other coordinates.
This result is given as \cref{lem:ZhangInfAtJeps}.
\begin{restatable}{proposition}{propZhangSymmetric}
\label{prop:ZhangSymmetric}
$\sloss{Zhang}$ satisfies \cref{cond:symmetry}.
\end{restatable}
\begin{proof}
This result is evident from the definition of $\sloss{Zhang}$ in \cref{table:scoreLosses}.
\end{proof}
\begin{restatable}{lemma}{lemZhangInfAtJeps}
\label{lem:ZhangInfAtJeps}
Consider $\sloss{Zhang}$ convex, with $\psi$ non-increasing and $\sn{S}$ satisfying \cref{ass:swapping}.
For all $\eps \geq 0$ and $p \in \yplex$,
\[
	\inf_{s \in \sn{T}(\sn{S}, \eps, p)} \risk^{\surr}(s,p) = \inf_{s \in \sn{M}(\sn{S}, \jeps)} \risk^{\surr}(s,p),
\]
where $\jeps \in \sn{J}(\eps,p)$.
\end{restatable}
\begin{proof}
We will use \cref{fact:infLowerBound} in this proof, so we must show that for every $s \in \sn{T}(\sn{S}, \eps, p)$ there exists $s' \in \sn{M}(\sn{S}, \jeps)$ \st{} $\risk^{\surr}(s', p) \leq \risk^{\surr}(s, p)$, where $\jeps \in \sn{J}(\eps,p)$.
By this argument, the result holds even if $\inf_{s \in \sn{T}(\sn{S}, \eps, p)} \risk^{\surr}(s,p) = -\infty$.

Fix $p \in \yplex$.
Take any $s \in \sn{T}(\sn{S}, \eps, p)$ \st{} $s_{\jeps} < s_{f(s)}$.
The definition of $\sn{J}(\eps,p)$ and the choice of $s$ imply that $p_{f(s)} \leq p_{\jeps}$.
Define $s'$ by
\als{
	s'_k \eqdef \begin{cases}
		s_{\jeps} & k = f(s), \\
		s_{f(s)} & k = \jeps, \\
		s_k & k \neq f(s),\jeps,
	\end{cases}
}
and note that $s' \in  \sn{T}(\sn{S}, \eps, p)$ by \cref{ass:swapping}.
Now, $F(\sum_{k=1}^{\K} \varphi(s_k)) = F(\sum_{k=1}^{\K} \varphi(s'_k))$, so
\als{
\risk^{\surr}(s,p) - \risk^{\surr}(s',p) &= p_{f(s)}(\psi(s_{f(s)}) - \psi(s'_{f(s)})) + p_{\jeps}(\psi(s_{\jeps}) - \psi(s'_{\jeps}))\\
&= (p_{\jeps} - p_{f(s)})(\psi(s_{\jeps}) - \psi(s_{f(s)})) \\
&\geq 0,
}
since $p_{\jeps} \geq p_{f(s)}$, $s_{\jeps} \leq s_{f(s)}$ and $\psi$ is non-increasing, therefore the result holds.
\end{proof}

\begin{restatable}{lemma}{lemZhangCondInfAtJepsJ0}
\label{lem:ZhangCondInfAtJepsJ0}
\Cref{cond:infAtJepsJ0} holds for $\sloss{Zhang}$ convex, with $\psi$ non-increasing and $\sn{S}$ satisfying \cref{ass:swapping,ass:averaging}.
\end{restatable}

\begin{proof}
Thanks to \cref{lem:ZhangInfAtJeps}, we only need to show that
\[
	\inf_{s \in \sn{M}(\sn{S}, \jeps)} \risk^{\surr}(s,p) = \inf_{s \in \sn{M}(\sn{S}, \jeps) \cap \sn{M}(\sn{S}, j_0)} \risk^{\surr}(s,p),
\]
for $\jeps \in \sn{J}(\eps, p)$ and $j_0 \in \sn{J}(0, p)$.
To that end, we will use \cref{fact:infLowerBound}, so we must show that for every $s \in \sn{M}(\sn{S}, \jeps)$ there exists $s'' \in \sn{M}(\sn{S}, \jeps) \cap \sn{M}(\sn{S}, j_0)$ \st{} $\risk^{\surr}(s'', p) \leq \risk^{\surr}(s, p)$.
By this argument, the result holds even if $\inf_{s \in \sn{M}(\sn{S}, \jeps)} \risk^{\surr}(s,p) = -\infty$.

Fix $p \in \yplex$.
Given $s \in \sn{M}(\sn{S}, \jeps)$, let $i$ \st{} $i \neq \jeps$ and $s_i \geq s_j$ for all $j \neq \jeps$.
Define $s' \in\sn{M}(\sn{S}, \jeps)$ (\see{} \cref{ass:swapping}) by
\als{
	s'_k \eqdef \begin{cases}
		s_{j_0}, & k = i \\
		s_i, & k = j_0 \\
		s_k & k \neq i,j_0.
	\end{cases}
}
Then $s'_{\jeps} \geq s'_{j_0} \geq s'_{j}$ for all $j \neq \jeps$, $\adj(s) = \adj(s')$ and
\als{
\risk^{\surr}(s, p) - \risk^{\surr}(s', p)
&= p_i(\psi(s_i) - \psi(s'_i)) + p_{j_0}(\psi(s_{j_0}) - \psi(s'_{j_0})) \\
&= (p_i - p_{j_0})(\psi(s_i) - \psi(s_{j_0})) \\
& \geq 0,
}
since $p_i \leq p_{j_0}$, $\psi$ is non-increasing and $s_i \geq s_{j_0}$.
Note also that $s' \in \sn{M}(\sn{S}, \jeps) \cap \sn{M}(\sn{S}, j_0)$ or $s'_{\jeps} > s'_{j_0} \geq s'_j$ for all $j \neq j_0,\jeps$.
If $s' \in \sn{M}(\sn{S}, \jeps) \cap \sn{M}(\sn{S}, j_0)$, we can take $s'' = s'$ and the result follows, otherwise define $s''$ by
\als{
	s''_k \eqdef \begin{cases}
		\frac{s'_{\jeps} + s'_{j_0}}{2} & k = \jeps \mbox{ or } k = j_0, \\
		s'_k & k \neq \jeps, j_0.
	\end{cases}
}
Then $s'' \in \sn{M}(\sn{S}, \jeps) \cap \sn{M}(\sn{S}, j_0)$ (by \cref{ass:averaging}), $s'_{\jeps} - s''_{\jeps} = s''_{j_0} - s'_{j_0}$, and by a convex lower-bound on $\risk^{\surr}(s', p)$ at $s''$,
\als{
\risk^{\surr}(s', p) - \risk^{\surr}(s'', p)
&\geq \pb{p_{\jeps}\psi'(s''_{\jeps}) + F'\pb{ \sum_{i=1}^{\K} \varphi(s''_i) }\varphi'(s''_{\jeps})}(s'_{\jeps} - s''_{\jeps}) \\
&\phantom{=}~+ \pb{p_{j_0}\psi'(s''_{j_0}) + F'\pb{ \sum_{i=1}^{\K} \varphi(s''_i) }\varphi'(s''_{j_0})}(s'_{j_0} - s''_{j_0}) \\
&= (p_{\jeps} - p_{j_0})\psi'(s''_{\jeps})(s'_{j_\eps} - s''_{j_\eps}) \\
& \geq 0,
}
which implies the result of \cref{lem:ZhangCondInfAtJepsJ0}.
For the last inequality above, we used that $\psi$ is non-increasing, $s'_{j_\eps} \geq s''_{j_\eps}$, and $p_{\jeps} \leq p_{j_0}$.
\end{proof}

\begin{restatable}{lemma}{lemZhangCondPairing}
\label{lem:ZhangCondPairing}
Under \cref{ass:lowerBounded}, \cref{cond:pairing} holds for $\sloss{Zhang}$ convex with $\psi$ non-increasing, $F(t) = t$, $\sn{S} = \rl{\K}$ and $\adj(s) = F(\sum_{k=1}^{\K} \varphi(s_k))$.
\end{restatable}

\begin{proof}
We have
\al{
	&\inf_{s \in \sn{M}(\sn{S}, i) \cap  \sn{M}(\sn{S}, j)}\sup_{\substack{s' \in \sn{S}:\\ s'_{k} = s_{k}, k \neq i,j}} \risk^{\surr}(s,p) - \risk^{\surr}(s',p) \notag \\
	&~~= \inf_{s \in \sn{M}(\sn{S}, i) \cap  \sn{M}(\sn{S}, j)}\sup_{\substack{s' \in \sn{S}:\\ s'_{k} = s_{k}, k \neq i,j}}\sum_{k=1}^{\K} p_k (\psi(s_k) - \psi(s'_k)) + \varphi(s_k) - \varphi(s'_k) \label[equation]{lem:ZhangCondPairing:2} \\
	&~~= \inf_{s \in \rl{}}\sup_{s' \in \rl{2}} p_{i} (\psi(s) - \psi(s'_1)) +  p_{j} (\psi(s) - \psi(s'_2)) + \adj((s,s)) - \adj(s') \label[equation]{lem:ZhangCondPairing:3} \\
	&~~= \inf_{s \in \sn{S}}\sup_{s' \in \sn{S}}\risk'(s,  (\bar{p}, \bar{p})) - \risk'(s', (p_i, p_j)). \label[equation]{lem:ZhangCondPairing:4}
}
In \eqref{lem:ZhangCondPairing:2}, we expanded the definition of $\risk^{\surr}$, and in \eqref{lem:ZhangCondPairing:3} we rewrote the objective and in \eqref{lem:ZhangCondPairing:4} we used \cref{prop:subOptimalSurrogateRisk}.
\end{proof}

At this point we can recover \cref{thm:zhang}, since $\sloss{RRKA}$ is equivalent to $\sloss{Zhang}$ with $\psi(t) = \varphi(-t)$ and $F(t) = t$.
The condition \eqref{eq:thm:zhang:condition} in \cref{thm:zhang} implies \cref{cond:zetaOfEps} with $\zeta(t) = c^2t^2$
, so \cref{lem:deltaMaxFramework,lem:ZhangCondInfAtJepsJ0,lem:ZhangCondPairing} give us \cref{thm:zhang}.

More generally, for other instances of $\sloss{Zhang}$, rather assuming that \eqref{eq:thm:zhang:condition} holds, we may want to impose conditions on $\psi$ and $\varphi$ that are easier to verify, and which also guarantee that \cref{cond:zetaOfEps} holds.
In order to do so, the structure of $\sloss{zhang}$ allows us to ``replace'' \cref{cond:zetaOfEps} with \cref{cond:ZhangInf}.
\Cref{cond:ZhangInf} is a purely technical condition that implies \cref{cond:zetaOfEps}, as shown by \cref{lem:ZhangCondInfWithP}.
\Cref{cond:ZhangInf} seems to be straightforward to verify if a closed form for $\delta_{\mathrm{binary}}$ is also easy to calculate, so we do not consider it too restrictive.
\begin{restatable}{condition}{condZhangInf}
\label{cond:ZhangInf}
With $\sloss{Zhang}$ as the surrogate loss, for all $p_1,p_2 \geq 0$,
\als{
	&\sup_{s' \in \sn{S}}\inf_{s \in \sn{S}} \risk'(s,  (\bar{p}, \bar{p})) - \risk'(s',  (p_1,p_2)) \\
	&~~= \sup_{s' \in \sn{S}}\inf_{\substack{s \in \sn{S}: \\ s_1 = s_2 \\ \psi(s_1) + \psi(s_2) \leq \psi(s'_1) + \psi(s'_2) }} \risk'(s,  (\bar{p}, \bar{p})) - \risk'(s',  (p_1,p_2))
}
where $\bar{p} = \frac{p_1 + p_2}{2}$.
\end{restatable}
\begin{restatable}{lemma}{lemZhangCondInfWithP}
\label{lem:ZhangCondInfWithP}
Consider $\sloss{Zhang}$ convex with $\psi$ non-increasing, $\adj(s) = F(\sum_{k=1}^{\K} \varphi(s_k))$ and $\sn{S}$ satisfying \cref{ass:swapping}.
If $L$ satisfies \cref{ass:lowerBounded}, then \cref{cond:ZhangInf} implies \cref{cond:infWithP}.
\end{restatable}

\begin{proof}
For all $p_1,p_2$ \st{} $p_1 \geq p_2 + \eps \geq 0$ and $p_1 + p_2 = 0$ we have
\al{
	&\inf_{s \in \sn{S}}\sup_{s' \in \sn{S}}\risk'(s,  (\bar{p}, \bar{p})) - \risk'(s', (p_1, p_2)) \notag \\
	&~~=\sup_{s' \in \sn{S}} \inf_{s \in \sn{S}}\risk'(s,  (\bar{p}, \bar{p})) - \risk'(s', (p_1, p_2)) \label[equation]{lem:ZhangCondInfWithP:2} \\
	&~~=\sup_{s' \in \sn{S}}\inf_{\substack{s \in \sn{S}: \\ s_1 = s_2 \\ 2\psi(s_1) \leq \psi(s'_1) + \psi(s'_2) }} (p_1 + p_2)\psi(s_1) + \adj(s) - p_1 \psi(s'_1) - p_2 \psi(s'_2) - \adj(s') \label[equation]{lem:ZhangCondInfWithP:3} \\
	\begin{split}
	\label[equation]{lem:ZhangCondInfWithP:4}
	&~~= \sup_{\substack{s' \in \sn{S}: \\ s'_1 \geq s'_2}}\inf_{\substack{s \in \sn{S}: \\ s_1 = s_2 \\ 2\psi(s_1) \leq \psi(s'_1) + \psi(s'_2) }} (p_1 + p_2)(\psi(s_1) - q\psi(s'_1) - (1 - q)\psi(s'_2)) \\
	&\phantom{~~=\sup_{\substack{s' \in \sn{S}: \\ s'_1 \geq s'_2}}\inf_{\substack{s \in \sn{S}: \\ s_1 = s_2 \\ 2\psi(s_1) \leq \psi(s'_1) + \psi(s'_2) }}}~~ + \adj(s) - \adj(s')
	\end{split} \\
	\begin{split}
	\label[equation]{lem:ZhangCondInfWithP:5}
	&~~\geq \sup_{\substack{s' \in \sn{S}: \\ s'_1 \geq s'_2}}\inf_{\substack{s \in \sn{S}: \\ s_1 = s_2 \\ 2\psi(s_1) \leq \psi(s'_1) + \psi(s'_2) }} \frac{1}{2}(p_1 + p_2)(2\psi(s_1) - \psi(s'_1) - \psi(s'_2)) \\
	&\phantom{~~\geq \sup_{\substack{s' \in \sn{S}: \\ s'_1 \geq s'_2}}\inf_{\substack{s \in \sn{S}: \\ s_1 = s_2 \\ 2\psi(s_1) \leq \psi(s'_1) + \psi(s'_2) }}}~~ - \frac{\eps}{2}(\psi(s'_1) - \psi(s'_2))  + \adj(s) - \adj(s')
	\end{split} \\
	&~~\geq \sup_{\substack{s' \in \sn{S}: \\ s'_1 \geq s'_2}}\inf_{\substack{s \in \sn{S}: \\ s_1 = s_2 \\ 2\psi(s_1) \leq \psi(s'_1) + \psi(s'_2) }} \psi(s_1) - \frac{1 + \eps}{2}\psi(s'_1) - \frac{1 + \eps}{2}\psi(s'_2) + \adj(s) - \adj(s') \label[equation]{lem:ZhangCondInfWithP:6} \\
	&~~= \inf_{s \in \sn{S}} \risk^{\surr}(s, p^0) - \inf_{s' \in \sn{S}} \risk^{\surr}(s', p^{\eps}) \label[equation]{lem:ZhangCondInfWithP:7},
}
where $q \eqdef \frac{p_1}{p_1 + p_2}$.
In \eqref{lem:ZhangCondInfWithP:2}, we rewrote the objective, and in \eqref{lem:ZhangCondInfWithP:3} we used \cref{cond:ZhangInf} and expanded $\risk'$.
To see that \eqref{lem:ZhangCondInfWithP:4} holds, it suffices to use \cref{fact:infLowerBound}, noting that for any $s'$ \st{} $s'_1 \leq s'_2$ we have
\[
	\risk'(s', (p_1, p_2)) - \risk'((s'_2,s'_1), (p_1, p_2))
	= (p_1 - p_2)(\psi(s'_1) - \psi(s'_2))
	\geq 0,
\]
since $\psi$ is non-increasing, $p_1 \geq p_2$ and $s'_1 \leq s'_2$.
Note that $s' \in \sn{M}(\sn{S},2)$ implies, by \cref{ass:swapping}, that $(s'_2,s'_1) \in \sn{M}(\sn{S}, 1)$.
In \eqref{lem:ZhangCondInfWithP:5} we used \cref{prop:pEps} and that $\psi(s'_1) - \psi(s'_2) \leq 0$.
To obtain \eqref{lem:ZhangCondInfWithP:6} we used that $2\psi(s_1) - \psi(s'_1) - \psi(s'_2) \leq 0$ along with $p_1 + p_2 \leq 1$.
To conclude, we used the argument of \eqref{lem:ZhangCondInfWithP:4} and  \cref{cond:ZhangInf} to arrive at \eqref{lem:ZhangCondInfWithP:7}.
\end{proof}

We summarize the results in this section into \cref{thm:lZhang}.
If the conditions of \cref{thm:lZhang} are satisfied for a particular choice of $\psi$ and $\varphi$, we get that the corresponding loss is calibrated iff $\delta_{\mathrm{binary}}(\eps) > 0$ for all $\eps > 0$.
Indeed, if $\delta_{\mathrm{binary}}$ is positive, then $\delta_{\max}$ must also be positive.
Conversely, if $\delta_{\mathrm{binary}}(\eps) = 0$ for some $\varepsilon$ then $\delta_{\max}(\varepsilon) = 0$ with $\K = 2$, in which case the surrogate loss cannot be calibrated.
\begin{restatable}{theorem}{thmLZhang}
\label{thm:lZhang}
Consider $L = \sloss{Zhang}$ convex with $\psi$ non-decreasing and $F(t) = t$.
If $L$ satisfies \cref{ass:lowerBounded,cond:ZhangInf}, then for all $\eps > 0$
\[
	\delta_{\max}(\eps) \geq \delta_{\mathrm{binary}}(\eps),
\]
and the above holds with equality when $\K = 2$.
\end{restatable}

\begin{proof}
The result follows immediately from \cref{prop:ZhangSymmetric,lem:ZhangCondInfAtJepsJ0,lem:ZhangCondPairing,lem:ZhangCondInfWithP,lem:deltaMaxFramework,lem:binaryTightness}.
\end{proof}

\Cref{table:decoupledLZhang} summarizes different calibrated variants of $\sloss{Zhang}$ discussed by \citet{zhang2004statistical}, and whether (and under which conditions) they satisfy \cref{cond:ZhangInf}.
We omit the standard calculations involved in verifying \cref{cond:ZhangInf} for each of the losses in \cref{table:decoupledLZhang}, and we also include the variant with $\psi(t) = -t$ and $\varphi(t) = (at + b)^2$ for $a,b \in \rl{}$.

\begin{table}[ht]
\begin{adjustwidth}{-1in}{-1in}
\centering
\begin{tabular}{ll>{\arraybackslash}l}
\toprule
$\psi(t)$ & $\varphi(t)$ & \Cref{cond:ZhangInf} \\
\midrule
$-t$ & $e^t$ & \checkmark \\
$-\ln t$ & $t$ & \checkmark \\
$-\frac{1}{a} t^a$ ($a \in (0,1)$) & $t$ & $a \in \left(0, \frac{1}{2} \right]$ \\
$-t$ & $\ln(1 + e^t)$ & \checkmark \\
$-t$ & $\frac{1}{a}\ab{t}^a$ ($a > 1$) & $a \geq 2$ \\
$-t$ & $\frac{1}{a}(t)^a_+$ ($a > 1$) & $a \geq 2$ \\
$-t$ & $(at + b)^2$ ($a,b \in \rl{}$) & \checkmark \\
\bottomrule
\end{tabular}
\end{adjustwidth}
\caption{Different variants of $\sloss{Zhang}$ with $F(t) = t$. \label{table:decoupledLZhang}}
\end{table}

\subsection{Multiclass logistic regression}
\label{sec:cases:logistic}

When the surrogate loss is $\sloss{Zhang}$ with $\psi(t) = -t$, $F(t) = \ln t$, $\varphi = \expt$ and $\sn{S} = \rl{\K}$, we obtain the multiclass logistic regression loss \citep{zhang2004statistical}.
From the results for the general $\sloss{Zhang}$, we already have  \cref{cond:infAtJepsJ0,cond:infWithP}, provided that we verify \cref{cond:ZhangInf}.
However, we need to show that \cref{cond:pairing} holds for this variant,
and doing so for a coupled formulation (where we do not have $F(t) = t$) can be challenging, mainly because the loss can no longer be expressed as a summation of $\K$ terms.
In order to address this issue for the logistic regression loss, we re-express it as a decoupled loss with $\sn{S} = \yplex$, which still does not satisfy the conditions of \cref{lem:ZhangCondPairing}, but is more amenable for verifying \cref{cond:pairing} directly.
The \emph{decoupled logistic regression loss} is $\fun{\sloss{LR}}{\yplex \times \yplex}{\rl{}}$ defined by
\[
	\sloss{LR}(s,y) \eqdef -\ln s_y,
\]
with $\sn{S} = \yplex$.

\Cref{lem:logisticSumToOne} shows that, with $\sloss{Zhang}$ as outlined above, minimizers $s^*$ of the surrogate risk over $\sn{T}(\sn{S}, \eps, p)$ for any $\eps > 0$ and $p \in \yplex$ satisfy $F(\sum_{k=1}^{\K}\varphi(s^*_k)) = 0$, effectively eliminating the coupled part of the loss.
Then \cref{prop:logisticSimplex} establishes the equivalence between the two forms of the logistic regression loss: coupled with $\sn{S} = \rl{\K}$, and decoupled with $\sn{S} = \yplex$.
\begin{restatable}{proposition}{propLogisticSumToOne}
\label{lem:logisticSumToOne}
Consider the surrogate loss $\sloss{Zhang}$ with $\psi(t) = -t$, $F(t) = \ln t$, $\varphi = \expt$, $\sn{S} = \rl{\K}$ and $\adj(s) = F\pb{\sum_{i=1}^{\K} \varphi(s_i)}$.
Assume $\sn{S}' \subset \sn{S}$ satisfies $s \in \sn{S}' \Rightarrow (s - \adj(s) \cdot \one{\K}) \in \sn{S}'$.
Then for any $p \in \yplex$,
\[
	\inf_{s \in \sn{S}'}\risk^{\surr}(s,p) = \inf_{\substack{s \in \sn{S}': \\ \adj(s) = 0}}\risk^{\surr}(s,p).
\]
\end{restatable}
\begin{proof}
We will use \cref{fact:infLowerBound} in this proof.
Take any $s \in \sn{S}'$ and define $s'$ by $s'_k \eqdef s_k - \adj(s)$.
Then $s' \in \sn{S}'$, $\adj(s') = 0$, $\sum_{k=1}^{\K} e^{s'_k} = 1$, and by a convex lower-bound on $\risk^{\surr}(s, p)$ at $s'$ (\see{} \cref{fact:convexLowerBound}),
\als{
\risk^{\surr}(s, p) - \risk^{\surr}(s', p)
&\geq \sum_{k=1}^{\K}\pb{-p_{k} + \frac{e^{s'_k}}{\sum_{i=1}^{\K} e^{s'_{i}} }}(s_k - s'_k) \\
&= \sum_{k=1}^{\K}\pb{-p_{k} + e^{s'_k} }\adj(s) \\
&= 0.
}
\end{proof}
\begin{restatable}{proposition}{propLogisticSimplex}
\label{prop:logisticSimplex}
Consider $\sloss{LR}$ with $\sn{S} = \yplex$ and $\sloss{Zhang}$ with $\psi(t) = -t$, $F(t) = \ln t$, $\expt$, $\sn{S} = \rl{\K}$ and $\adj(s) = F\pb{\sum_{i=1}^{\K} \varphi(s_i)}$.
Then for any $p \in \yplex$ and $j \in \range{\K}$
\[
	\inf_{s \in \sn{M}(\rl{\K},j)}\risk^{\surr}_{\sloss{Zhang}}(s,p) = \inf_{s \in \sn{M}(\simplex{\K},j)}\risk^{\surr}_{\sloss{LR}}(s,p).
\]
\end{restatable}
\begin{proof}
The result follows trivially by combining \cref{lem:logisticSumToOne} and by observing that for each $s' \in \cb{ s \in \sn{M}(\rl{\K}, j) }$ we have $s'' \in  \sn{M}(\simplex{\K}, j)$, when $s''_k \eqdef e^{s'_k}$.
\end{proof}
Thanks to \cref{prop:logisticSimplex}, we can verify \cref{cond:pairing} for $\sloss{LR}$, the decoupled formulation of the logistic regression loss, which is done in \cref{lem:logisticCondPairing}.
\Cref{lem:logisticCondInfWithP} verifies \cref{cond:ZhangInf}, which combined with \cref{lem:ZhangCondInfWithP} gives us \cref{cond:infWithP}.
\begin{restatable}{lemma}{lemLogisticCondPairing}
\label{lem:logisticCondPairing}
Under \cref{ass:lowerBounded}, \cref{cond:pairing} holds for the surrogate loss $\sloss{LR}$ with score set $\sn{S} = \yplex$.
\end{restatable}
\begin{proof}
We have
\al{
	&\inf_{s \in \sn{M}(\sn{S}, i) \cap \sn{M}(\sn{S}, j)}\sup_{\substack{s' \in \sn{S}:\\ s'_{k} = s_{k}, k \neq i,j}} \risk^{\surr}(s,p) - \risk^{\surr}(s',p) \notag \\
	&~~= \inf_{0 \leq s \leq 1}\sup_{ \substack{s'_1, s'_2 \geq 0 : \\ s'_1 + s'_2 = s} } -(p_1 + p_2)\ln\frac{s}{2} + p_1 \ln s'_1 + p_2 \ln s'_2 \label[equation]{lem:logisticCondPairing:2} \\
	&~~= \inf_{0 \leq s \leq 1}\sup_{ \substack{s'_1, s'_2 \geq 0 : \\ s'_1 + s'_2 = s} } -(p_1 + p_2)\ln\frac{1}{2} + p_1 \ln \frac{s'_1}{s'_1 + s'_2} + p_2 \ln \frac{s'_2}{s'_1 + s'_2} \label[equation]{lem:logisticCondPairing:3} \\
	&~~= \sup_{ \substack{s'_1, s'_2 \geq 0 : \\ s'_1 + s'_2 = 1} } -(p_1 + p_2)\ln\frac{1}{2} + p_1 \ln s'_1 + p_2 \ln s'_2 \label[equation]{lem:logisticCondPairing:4} \\
	&~~= \sup_{ s' \in \simplex{2} } -(p_1 + p_2)\ln\frac{1}{2} + p_1 \ln s'_1 + p_2 \ln s'_2 \label[equation]{lem:logisticCondPairing:5} \\
	&~~= \inf_{s \in \sn{S}} \risk'(s, (\bar{p}, \bar{p})) - \inf_{s' \in \sn{S}} \risk'(s', (p_1, p_2)). \label[equation]{lem:logisticCondPairing:6}
}
In \eqref{lem:logisticCondPairing:2} we expanded the definition of $\sloss{LR}$ and rewrote the objective.
In \eqref{lem:logisticCondPairing:3} we rewrote the objective using that $s = s'_1 + s'_2$, while in \eqref{lem:logisticCondPairing:4} we dropped the infimum and rewrote the objective.
In \eqref{lem:logisticCondPairing:5} we also rewrote the objective and in \eqref{lem:logisticCondPairing:6} we plugged in the definition of $\risk'$ after using the fact that
\[
	\inf_{s \in \simplex{2}} -\frac{1}{2}(p_1 + p_2) (\ln s_1 + \ln s_2) =  -(p_1 + p_2)\ln \frac{1}{2}.
\]
To see that the above does indeed hold, it suffices to note that, by concavity, $\ln s_1 + \ln s_2 \leq 2\ln \frac{s_1 + s_2}{2} = 2\ln \frac{1}{2}$, so that the infimum must be taken at $s = \pb{\frac{1}{2}, \frac{1}{2}}$.
\end{proof}
\begin{restatable}{lemma}{lemLogisticCondInfWithP}
\label{lem:logisticCondInfWithP}
Under \cref{ass:lowerBounded}, \cref{cond:ZhangInf} holds for the surrogate loss $\sloss{LR}$ with $\sn{S} = \yplex$.
\end{restatable}
\begin{proof}
This result follows immediately from \cref{prop:logisticInf} and the following three facts: \emph{(i)} the infimum in \cref{cond:ZhangInf} is taken at $s = \pb{\frac{1}{2}, \frac{1}{2}}$, the supremum at $s' = \pb{\frac{p_1}{p_1 + p_2}, \frac{p_2}{p_1 + p_2}}$; \emph{(ii)} $s_1 = \frac{s'_1 + s'_2}{2}$; and \emph{(iii)} $\psi$ is convex.
\end{proof}
We conclude this section with the calibration function for the logistic regression loss.
\begin{restatable}{theorem}{thmLogistic}
\label{thm:logistic}
Consider $\sloss{LR}$ with $\sn{S} = \yplex$, or, equivalently, $\sloss{Zhang}$  with $\psi(t) = -t$, $F(t) = \ln t$, $\varphi = \expt$, $\sn{S} = \rl{\K}$.
Then for all $\eps > 0$
\[
	\delta_{\max}(\eps) \geq \delta_{\mathrm{binary}}(\eps),
\]
and the above holds with equality when $\K = 2$.
\end{restatable}

\begin{proof}
The equivalence between the surrogate losses follows from \cref{prop:logisticSimplex}.
The main statement follows from \cref{prop:ZhangSymmetric,lem:ZhangCondInfAtJepsJ0,lem:logisticCondPairing,lem:ZhangCondInfWithP,lem:deltaMaxFramework,lem:binaryTightness} and from noting that \cref{ass:lowerBounded} is satisfied because $\risk^{\surr}_{\sloss{LR}}$ is non-negative for all $s,p \in \yplex$.
\end{proof}

\subsection{The loss of \citet{lee2004multicategory}}
\label{sec:cases:S0}

In this section we develop a variant of \cref{thm:avilapires} using the streamlined analysis, that is, we instantiate our analysis with $\sloss{LLW}$ and $\sn{S} = \sn{S}_0$.
In order to recover \cref{thm:avilapires}, we need to verify \cref{cond:infAtJepsJ0,cond:pairing} as well as \cref{cond:zetaOfEps} with the corresponding $\zeta$.
For ease of presentation, we will verify the conditions for a special case of \cref{thm:avilapires}, where \cref{cond:pairing} is verified via \cref{cond:pairingSumToZero,cond:sumToZeroFreeLunch}, the letter of which will be assumed to hold.
We will discuss choices of $\varphi$ for which \cref{cond:sumToZeroFreeLunch} does hold, and pose a conjecture about more general $\varphi$ for which the condition also holds.
Because we are concerned with the cost-insensitive setting (differently from \citet{avilapires2013costsensitive}), we are able to refine \cref{thm:avilapires} by verifying \cref{cond:infWithP} and provide a result for $\delta_{\max}$ independent of $p$.

First, we establish through \cref{prop:LLWSymmetric} that \cref{cond:symmetry} is satisfied.
\Cref{cond:symmetry} is used to verify that \cref{cond:infAtJepsJ0} holds (\see{} \cref{lem:LLWCondInfAtJepsJ0}), with the swapping argument used in \cref{sec:cases:Zhang}, to show the same condition for $\sloss{Zhang}$.
By writing $\sloss{LLW}$ with an appropriate adjustment, we can reuse \cref{lem:ZhangCondInfAtJepsJ0} to prove \cref{lem:LLWCondInfAtJepsJ0}.
Note, however, that in \cref{lem:LLWCondInfAtJepsJ0} we are restricted to $\varphi$ non-decreasing.
This restriction allows us to use the swapping argument, and it will be ``removed'' using \cref{lem:lowerBoundedScores} of \citet{avilapires2013costsensitive}, which shows that $\sloss{LLW}$ with $\varphi$ has the same $\delta_{\max}$ to $\sloss{LLW}$ with a corresponding non-decreasing $\varphi$.
\begin{restatable}{proposition}{propLLWSymmetric}
\label{prop:LLWSymmetric}
$\sloss{LLW}$ satisfies \cref{cond:symmetry}.
\end{restatable}
\begin{proof}
This result is evident from the definition of $\sloss{LLW}$ (\see{} \cref{table:scoreLosses}).
\end{proof}
\begin{restatable}{lemma}{lemLLWCondInfAtJepsJ0}
\label{lem:LLWCondInfAtJepsJ0}
\Cref{cond:infAtJepsJ0} holds for $\sloss{LLW}$ with $\sn{S} = \sn{S}_0$ and $\varphi$ convex non-decreasing.
\end{restatable}
\begin{proof}
This condition holds thanks to \cref{lem:ZhangCondInfAtJepsJ0}, since $\sloss{LLW}$ is equivalent to $\sloss{Zhang}$ with $\psi(t) = -\varphi(t)$ and $F(t) = t$, where $\varphi$ non-decreasing implies $\psi$ non-increasing.
\end{proof}
As mentioned in \cref{sec:streamlining}, we verify \cref{cond:pairing} via \cref{cond:pairingSumToZero,cond:sumToZeroFreeLunch}.
\Cref{lem:LLWCondPairingSumToZero} establishes \cref{cond:pairingSumToZero}.
The proof is based on a straightforward argument that mirrors the example in \cref{sec:streamlining}~\eqref{eq:exampleObjective:fixingCoordinates}, while accounting for the requirement that scores be in $\sn{S}_0$.
\begin{restatable}{lemma}{lemLLWCondPairingSumToZero}
\label{lem:LLWCondPairingSumToZero}
Under \cref{ass:lowerBounded}, \cref{cond:pairingSumToZero} holds for $\sloss{LLW}$ with $\sn{S} = \sn{S}_0$, $\varphi$ convex and $\adj(s) = \sum_{k=1}^{\K} \varphi(s_k)$.
\end{restatable}
\begin{proof}
We have
\al{
	&\inf_{s \in \sn{M}(\sn{S}, i) \cap  \sn{M}(\sn{S}, j)}\sup_{\substack{s' \in \sn{S}:\\ s'_{k} = s_{k}, k \neq i,j}} \risk^{\surr}(s,p) - \risk^{\surr}(s',p) \notag \\
	&~~= \inf_{s \in \sn{M}(\sn{S}, i) \cap  \sn{M}(\sn{S}, j)}\sup_{\substack{s' \in \sn{S}:\\ s'_{k} = s_{k}, k \neq i,j}}\sum_{k=1}^{\K} (1 - p_k) (\varphi(s_k) - \varphi(s'_k)) \label[equation]{lem:LLWCondPairingSumToZero:2} \\
	&~~= \inf_{s \geq 0}\sup_{\substack{s' \in \rl{2}:\\ s'_1 + s'_2 = 2s}} (1 - p_{i})(\varphi(s) - \varphi(s'_1)) + (1 - p_{j})(\varphi(s) - \varphi(s'_2)) \label[equation]{lem:LLWCondPairingSumToZero:3} \\
	&~~= \inf_{s \geq 0}\sup_{s' \in \rl{}}\risk'((s,s),  (\bar{p}, \bar{p})) - \risk'((s + s', s - s'), (p_i, p_j)). \label[equation]{lem:LLWCondPairingSumToZero:4}
}
In \eqref{lem:LLWCondPairingSumToZero:2}, we expanded the definition of $\risk^{\surr}$.
We rewrote the objective in \eqref{lem:LLWCondPairingSumToZero:3}, where the sum-to-zero constraint and the choice of $s'$ in \eqref{lem:LLWCondPairingSumToZero:2} requires us to have $s'_i + s'_j = s_i + s_j$.
In \eqref{lem:LLWCondPairingSumToZero:4} we rewrote the objective and used \cref{prop:subOptimalSurrogateRisk}.
\end{proof}
It remains to verify \cref{cond:sumToZeroFreeLunch}, which is challenging to do in general.
\Cref{conj:sumToZeroFreeLunch} states that continuously differentiable $\varphi$ allow \cref{cond:sumToZeroFreeLunch} to hold.
We have not been able to prove or disprove this fact, but continuous differentiability in the increasing part of $\varphi$ (in addition to convexity and $\varphi'(0) > 0$) seems to be a minimal requirement, since adding kinks can lead to non-calibrated losses, as \cref{prop:kinkNotCalibrated} suggests.
On the other hand, \cref{cond:sumToZeroFreeLunch} can be shown to hold for $\hinge$, $\expt$, $\squared$, while such result is not clear for $\logit$.
We present a proof for $\hinge$ ahead, but we omit the others, which can be obtained with straightforward calculations\footnote{
	With $\expt$, $\exp^{\theta}$ can be factored out, and with $\squared$ the result can be obtained by explicit calculation of the infima and supremum in \cref{cond:pairingSumToZero}.
}.
\begin{conjecture}
\label{conj:sumToZeroFreeLunch}
Assume \cref{ass:lowerBounded} and let
\[
	\sinf \eqdef \sup\cb{ t : \varphi(t) \leq \varphi(t'), \forall t' \in \rl{} }.
\]
If $\varphi$ is differentiable in $(\sinf, -\sinf)$, then \cref{cond:sumToZeroFreeLunch} holds.
\end{conjecture}
\begin{restatable}{lemma}{condSumToZeroFreeLunchHinge}
\label{lem:sumToZeroFreeLunchHinge}
The loss $\sloss{LLW}$ with $\sn{S} = \sn{S}_0$ satisfies \cref{cond:sumToZeroFreeLunch}, $\varphi = \hinge$ and $\adj(s) = \sum_{k=1}^{\K} \varphi(s_k)$.
\end{restatable}
\begin{proof}
Assume, without loss of generality, that $p_1 \geq p_2$.
Then
\al{
	&\inf_{s \geq 0}\sup_{s' \in \rl{}}\risk'((s,s),  (\bar{p}, \bar{p})) - \risk'((s + s', s - s'), (p_1, p_2)) \notag \\
	&~~= \inf_{s \geq 0}\sup_{s' \in \rl{}}(2 - p_1 - p_2)(1 + s)_+ - (1 - p_1)(1 + s + s')_+ - (1 - p_2)(1 + s - s')_+ \label[equation]{lem:sumToZeroFreeLunchHinge:2} \\
	&~~= \inf_{s \geq 0}(2 - p_1 - p_2)(1 + s) - 2(1 - p_1)(1 + s) \label[equation]{lem:sumToZeroFreeLunchHinge:3} \\
	&~~\geq (2 - p_1 - p_2) - 2(1 - p_1) \label[equation]{lem:sumToZeroFreeLunchHinge:4} \\
	&~~= \inf_{s \in \rl{}} \risk'((s,s), (\bar{p},\bar{p})) - \inf_{s' \in \sn{S}} \risk'((s'_1, s'_2), (p_1, p_2)) \label[equation]{lem:sumToZeroFreeLunchHinge:5}
}
In \eqref{lem:sumToZeroFreeLunchHinge:2} we expanded the definition of $\risk'$, and in \eqref{lem:sumToZeroFreeLunchHinge:3} we used the fact that the supremum is taken at $s' = s + 1$ (since $p_1 \geq p_2$).
To obtain \eqref{lem:sumToZeroFreeLunchHinge:4}, we used that $p_1 \geq p_2$, so the infimum must be taken at $s = 0$.
Then all that remains to be seen is that the infimum in \eqref{lem:sumToZeroFreeLunchHinge:5} is taken at $s = 0$ and that the supremum is taken at $s' = 1$.
\end{proof}
The argument used to prove \cref{lem:LLWCondInfWithP}, which verifies \cref{lem:LLWCondInfWithP}, also relies on $\varphi$ being non-decreasing, 
\begin{restatable}{lemma}{lemLLWCondInfWithP}
\label{lem:LLWCondInfWithP}
Under \cref{ass:lowerBounded}, \cref{cond:infWithP} holds for $\sloss{LLW}$ with $\sn{S} = \sn{S}_0$, $\varphi$ convex non-decreasing, $\adj(s) = \sum_{k=1}^{\K} \varphi(s_k)$
\end{restatable}
\begin{proof}
\al{
	&\sup_{s' \in \sn{S}}\inf_{s \in \sn{S}} \risk'(s,  (\bar{p}, \bar{p})) - \risk'(s',  (p_1,p_2)) \notag \\
	&~~=\sup_{s' \in \sn{S}} \risk'(\zero{2},  (\bar{p}, \bar{p})) - \risk'(s',  (p_1,p_2)) \label[equation]{lem:LLWCondInfWithP:2} \\
	&~~= (2 - p_1 - p_2) \sup_{s' \geq 0} \varphi(0) - (1 - p_1)\varphi(s') - (1 - p_2)\varphi(-s') \label[equation]{lem:LLWCondInfWithP:3} \\
	&~~= (2 - p_1 - p_2)\sup_{s' \geq 0}  \varphi(0) - (1 - q)\varphi(s') - q\varphi(-s') \label[equation]{lem:LLWCondInfWithP:4} \\
	&~~= (2 - p_1 - p_2)\sup_{s' \geq 0}  \varphi(0)  - \varphi(s') + q(\varphi(s') - \varphi(-s')) \label[equation]{lem:LLWCondInfWithP:5} \\
	&~~\geq  (2 - p_1 - p_2) \sup_{s' \geq 0} \varphi(0) - \frac{1}{2}\varphi(-s') - \frac{1}{2}\varphi(s') + \frac{\eps}{2} \cdot \frac{1}{2 - p_1 - p_2} (\varphi(s') - \varphi(-s')) \label[equation]{lem:LLWCondInfWithP:6} \\
	&~~\geq \sup_{s' \geq 0} \varphi(0) - \varphi(-s') - \varphi(s') + \frac{\eps}{2} (\varphi(s') - \varphi(-s')) \label[equation]{lem:LLWCondInfWithP:7} \\
	&~~= \inf_{s \in \sn{S}} \risk^{\surr}(s, p^0) - \inf_{s' \in \sn{S}} \risk^{\surr}(s', p^{\frac{\eps}{2}}) \label[equation]{lem:LLWCondInfWithP:8}
}
Having noted that \cref{cond:infAtJepsJ0} holds, we used \cref{prop:subOptimalSurrogateRisk} in \eqref{lem:LLWCondInfWithP:2}, and in \eqref{lem:LLWCondInfWithP:3} we expanded the definition of $\risk'$ and used that for all $s' \leq 0$ we have
\als{
	&(1 - p_1)(\varphi(s') - \varphi(-s')) + (1 - p_2)(\varphi(-s') - \varphi(s')) \\
	&~~= (p_1 - p_2)(\varphi(-s') - \varphi(s')) \geq 0,
}
since $\varphi$ is non-decreasing and $p_1 \geq p_2$.
In \eqref{lem:LLWCondInfWithP:4} and \eqref{lem:LLWCondInfWithP:5} we rewrote the objective with $q \eqdef \frac{1 - p_2}{2 - p_1 - p_2}$, and in \eqref{lem:LLWCondInfWithP:6} we used \cref{prop:pEps} (with $1 - p_2$ as $p_1$ and $1 - p_1$ as $p_2$) combined with the fact that $\varphi(s') - \varphi(-s') \geq 0$.
In \eqref{lem:LLWCondInfWithP:6} we used that $2 - p_1 - p_2 \geq 1$ combined with the fact that the objective is non-negative, and we also used that
\[
	\frac{1}{2 - p_1 - p_2} \geq \frac{1}{2 - p_2 - \eps} \geq \frac{1}{2 - \eps} \geq \eps.
\]
The first and second inequality hold because $0 \leq p_2 \leq p_1 - \eps$, and the third inequality holds iff $(2 - \eps)\eps \geq 1$ (since $\eps \leq 1$), which holds iff $(\eps - 1)^2 \geq 0$, which always holds.
Finally, in \eqref{lem:LLWCondInfWithP:7} we rewrote the objective, while plugging in the definition of $\risk^{\surr}$ and lower-bounding the whole quantity (which is non-negative) by its half.
\end{proof}
In order to apply \cref{lem:LLWCondInfAtJepsJ0,lem:LLWCondInfWithP} with $\varphi$ convex, but not necessarily non-decreasing, we can use the following result by \citet{avilapires2013costsensitive} \citep[see also]{zou2006margin}.
\begin{lemma}[\citealt{avilapires2013costsensitive}]
\label{lem:lowerBoundedScores}
Consider the surrogate loss $\sloss{LLW}$ with $\sn{S} = \sn{S}_0$, $\varphi$ convex \st{} $\varphi'(0) > 0$, and $\adj(s) = \sum_{k=1}^{\K} \varphi(s_k)$.
Let
\[
	\sinf \eqdef \sup\cb{ t : \varphi(t) \leq \varphi(t'), \forall t' \in \rl{} }.
\]
If $\varphi$ is non-decreasing, let $\sigma = \varphi$, otherwise let $\sigma(t) \eqdef \varphi(\sinf) + (\varphi(t) - \varphi(\sinf))_+$.
Under \cref{ass:lowerBounded}, for all $p \in \rl{\K}_+$ and all $j \in \range{\K}$ we have
\[
	\inf_{s \in \sn{M}(\sn{S}, j)} \risk'_{\sloss{LLW}_\varphi}(s, p)
	= \inf_{\substack{s \in \sn{M}(\sn{S}, j): \\ s \geq \sinf \cdot \one{\K}}} \risk'_{L^\varphi}(s, p)
	= \inf_{s \in \sn{M}(\sn{S}, j)} \risk'_{\sloss{LLW}_\sigma}(s, p).
\]
\end{lemma}

We are ready to present \cref{thm:lLLW}, a refined version of \cref{thm:avilapires} for the case where the infimum in \eqref{eq:deltaMaxTheta} is taken at $\theta = 0$.
While in \cref{thm:avilapires} $\delta_{\max}$ and $\delta_{\mathrm{binary}}$ depend on $p$, this is not the case in \cref{thm:lLLW}.
\begin{restatable}{theorem}{thmLLLW}
\label{thm:lLLW}
Consider $\sloss{LLW}$ with $\sn{S} = \sn{S}_0$ and $\varphi$ convex, lower-bounded.
If \cref{cond:sumToZeroFreeLunch} holds, then for all $\eps > 0$
\[
	\delta_{\max}(\eps) \geq \delta_{\mathrm{binary}}(\eps),
\]
and the above holds with equality when $\K = 2$.
\end{restatable}

\begin{proof}
We have that $\inf_{s \in \sn{S}_0, p \in \yplex} \risk^{\surr}(s,p) > -\infty$ iff $\inf_t \varphi(t) > -\infty$, so $\varphi$ lower-bounded implies \cref{ass:lowerBounded}.
The result then follows by applying \cref{lem:lowerBoundedScores} combined with \cref{prop:LLWSymmetric,lem:LLWCondInfAtJepsJ0,lem:LLWCondPairingSumToZero,lem:LLWCondInfWithP} applied to $\sigma$ as defined in \cref{lem:lowerBoundedScores}.
\Cref{lem:binaryTightness} is used to obtain the statement for the case of $\K = 2$.
\end{proof}

We note in passing that from \cref{thm:lLLW,lem:lowerBoundedScores} we can also obtain a result for $\sloss{Liu}$, once we realize that it has the same $\delta_{\max}$ as $\sloss{LLW}$ with $\hinge$.
\begin{restatable}{corollary}{corollarylLiu}
\label{cor:lLiu}
Consider $\sloss{Liu}$.
For all $\eps > 0$ we have
\[
	\delta_{\max}(\eps) \geq \delta_{\mathrm{binary}}(\eps),
\]
and the above holds with equality when $\K = 2$.
\end{restatable}

\begin{proof}
Consider $\sloss{LLW}$ with $\hinge$ for which \cref{thm:lLLW,cond:sumToZeroFreeLunch} hold, thanks to \cref{lem:sumToZeroFreeLunchHinge}.
We can apply \cref{lem:lowerBoundedScores} (with $\sinf = -1$) to see that for all $j \in \range{\K}$
\als{
	\inf_{s \in \sn{M}(\sn{S}, j)} \risk^{\surr}_{\sloss{LLW}}(s, p)
	&~~= \inf_{\substack{s \in \sn{M}(\sn{S}, j): \\ s \geq -1 \cdot \one{\K}}} \risk^{\surr}_{\sloss{LLW}}(s, p) \\
	&~~= \inf_{\substack{s \in \sn{M}(\sn{S}, j): \\ s \geq -1 \cdot \one{\K}}} \sum_{k = 1}^{\K} (1 - p_k)(1 + s_k)_+ \\
	&~~= \inf_{\substack{s \in \sn{M}(\sn{S}, j): \\ s \geq -1 \cdot \one{\K}}} \sum_{k = 1}^{\K} (1 - p_k)(1 + s_k) \\
	&~~=  \K - 1 + \inf_{\substack{s \in \sn{M}(\sn{S}, j): \\ s \geq -1 \cdot \one{\K}}} \sum_{k = 1}^{\K} - p_k s_k \\
	&~~= 1 + \inf_{\substack{s \in \sn{M}(\sn{S}, j): \\ s \geq -1 \cdot \one{\K}}} \sum_{k = 1}^{\K} p_k (\K - 2 - s_k) \\
	&~~= 1 + \inf_{\substack{s \in \sn{M}(\sn{S}, j): \\ s \geq -1 \cdot \one{\K}}} \sum_{k = 1}^{\K} p_k (\K - 2 - s_k)_+ \\
	&~~= 1 + \inf_{s \in \sn{M}(\sn{S}, j)} \risk'_{\sloss{Liu}}(s, p),
}
where we have used that $\sum_{k=1}^{\K} (1 - p_k) = \K - 1$ and that $\sum_{k=1}^{\K} s_k = 0$ since $\sn{S} = \sn{S}_0$.
Evidently, adding a constant to a loss does not alter $\delta_{\max}$, so the inequality in the statement follows.
To conclude, \cref{lem:binaryTightness} is used to obtain the statement for the case of $\K = 2$.
\end{proof}

\section{Conclusion}
\label{sec:conclusion}

In this paper, we refined a strategy to lower-bound $\delta_{\max}$ for multiclass calibration functions, which can be challenging to calculate for common surrogate loss choices.
The strategy presented reduces multiclass classification functions ($\delta_{\max}$) to binary-like calibration functions ($\delta_{\mathrm{binary}}$), which are often simple to instantiate for different loss choices (\see{} \cref{lem:deltaMaxFramework}).
As as an additional advantage, reducing calibration functions to $\delta_{\mathrm{binary}}$ gives us improved calibration guarantees (not for the pointwise risk, but for the risk) under the Mammen-Tsybakov noise condition, as shown in \cref{thm:MTNC}.
\Cref{thm:MTNC} generalizes Theorem 3 of \citet{bartlett2006convexity} to the multiclass case.

Our ``reduction'' strategy requires that we verify a set of conditions that break down, in a general way, the complexity of lower-bounding $\delta_{\max}$ for specific surrogate losses.
To illustrate the generality of our analysis, we instantiated it for different losses, including different variants of $\sloss{Zhang}$ with $F(t) = t$ (the so-called decoupled formulations, \see{} \cref{thm:lZhang}), the logistic regression loss (also special case of $\sloss{Zhang}$, \see{} \cref{thm:logistic}), and $\sloss{Liu}$ (\see{} \cref{cor:lLiu}).
We also used our analysis to recover previously-existing results for the one-versus-all loss ($\sloss{RRKA}$, recovered by \cref{thm:zhang}) and some variants of $\sloss{LLW}$ (\see{} \cref{thm:lLLW}).
Our results for $\sloss{LLW}$ are a refinement over existing results because they provide a bound on $\delta_{\max}(\eps)$, not on $\delta_{\max}(\eps, p)$  (\cf{} \cref{thm:avilapires,thm:lLLW}), and have weaker conditions.

Therefore, we both presented novel results for a large family of surrogate losses, including the logistic regression loss, which has been frequently used in practice.
Moreover, we recovered and refined existing results in a unified manner, and generalized results that give us improved calibration guarantees for the risk under the Mammen-Tsybakov noise condition.
More importantly, our results streamline the process of deriving calibration functions for different families of losses, and provides a path for obtaining similar results for losses proposed in the future.

The reduction analysis covers a majority of the multiclass surrogate losses presented, but not all: $\sloss{WW}$ and $\sloss{BSKV}$ require care when verifying \cref{cond:pairing}, and doing so is left for future work.
The loss $\sloss{ZZH}$ can be seen to satisfy \cref{cond:infAtJepsJ0,cond:infWithP} (by expressing it as $\sloss{Zhang}$), but having $\sn{S} = \sn{S}_0$ requires us to have some care in verifying \cref{cond:pairing}.
It can be seen that \cref{cond:sumToZeroFreeLunch} fails to hold even in simple cases, \eg{}, with $\expt$, so minor adaptations of the results for $\sloss{LLW}$ does not seem feasible either.
To make matters worse, $\sloss{ZZH}$ is not calibrated for $p \in \yplex \setdiff \yplex^\circ$, which is reflected by the fact that $\delta_{\max}(\eps) = 0$, so any calibration functions for $\sloss{ZZH}$ must apply only to $p \in \yplex^\circ$ and the definition of $\delta_{\max}(\eps)$ needs to be adjusted.

We know that $\sloss{WW}$, $\sloss{BSKV}$ with $F(t) = t$ and $\sloss{ZZH}$ all reduce to a margin loss when $\K = 2$, which is consistent, in particular, with $\hinge$.
However, we also know that none of these losses is calibrated with $\hinge$ in the multiclass case, so any instantiation of the reduction analysis would have to take the conditions for calibration of these losses into account.
Order-preservation and/or strong convexity are sensible conditions to impose on $\varphi$, if we consider \cref{thm:zhang}.
Therefore, it is likely that proving \cref{cond:pairing} for any of these losses would require such conditions, if the reduction analysis is suitable at all.

Regarding \cref{cond:ZhangInf}, we have seen that it does not hold for all the decoupled formulations discussed by \citet{zhang2004statistical} (\see{} \cref{table:decoupledLZhang}).
It would be useful to seek a relaxation of \cref{cond:ZhangInf} as well; due to the technical nature of \cref{cond:ZhangInf}, one ought to seek alternative proofs for \cref{lem:ZhangCondInfWithP}, and the losses in \cref{table:decoupledLZhang} would be a good starting point.
Additionally, understanding under which conditions on $F$ we can verify \cref{cond:pairing} for $\sloss{Zhang}$ would help us apply the reduction analysis to coupled formulations other than the logistic regression loss.

We have also left a series of topics aside in this work, but they are worthy of investigation.
It is important to investigate the interaction between \cref{ass:surrogateRiskBound} and calibration functions.
For example, if the surrogate loss is $\sloss{LLW}$ with $\hinge$, we get from \cref{thm:lLLW,table:varphi} that $\delta(\eps) = \eps$ is a calibration function for the loss, and if we use the same result for $L(s,p) = \frac{1}{c}\sloss{LLW}(s,p)$ we get the calibration function $\frac{\eps}{c}$.
This scaling introduces a factor of $c$ to the risk bound obtained through \cref{thm:steinwart}, while one might want to have $c = \K$ so the scale of $L(s,p)$ does not depend on $\K$.
The example used in \cref{sec:streamlining}, $\sloss{Zhang}$ with $\psi(t) = -t$, $F(t) = t$, $\varphi = \expt$ and $\sn{S} = \rl{K}$ is a good starting point for understanding the influence of scaling on the calibration functions, since we can easily calculate the surrogate risk minimizers that are restricted to be $\varepsilon$-suboptimal.

We have also left aside the possible extensions of \cref{thm:avilapires} to simplex coding \citep{mroueh2012multiclass,avilapires2013costsensitive}.
It would be interesting to better understand the role of simplex coding and other score transformations for losses other than $\sloss{LLW}$.
These transformations may be ultimately generalized by the unified formulation introduced by \citet{dougan2016unified}, which should also be considered.

\section*{Acknowledgments}

This work was supported by Alberta Innovates Technology Futures and NSERC.

{
    \bibliographystyle{apalike}
    \bibliography{bibliography}

\begin{thebibliography}{}

\bibitem[\'Avila~Pires et~al., 2013]{avilapires2013costsensitive}
\'Avila~Pires, B., Szepesv\'ari, C., and Ghavamzadeh, M. (2013).
\newblock Cost-sensitive multiclass classification risk bounds.
\newblock In Dasgupta, S. and McAllester, D., editors, {\em Proceedings of The
  30\textsuperscript{th} International Conference on Machine Learning},
  volume~28 of {\em JMLR: Workshop \& Conference Proceedings (ICML'13)}, pages
  1391--1399.

\bibitem[Bartlett et~al., 2006]{bartlett2006convexity}
Bartlett, P.~L., Jordan, M.~I., and McAuliffe, J.~D. (2006).
\newblock Convexity, classification, and risk bounds.
\newblock {\em Journal of the American Statistical Association},
  101(473):138--156.

\bibitem[Beijbom et~al., 2014]{beijbom2014guessaverse}
Beijbom, O., Saberian, M., Kriegman, D., and Vasconcelos, N. (2014).
\newblock Guess-averse loss functions for cost-sensitive multiclass boosting.
\newblock In Xing, E.~P. and Jebara, T., editors, {\em Proceedings of The
  31\textsuperscript{st} International Conference on Machine Learning},
  volume~32 of {\em JMLR: Workshop \& Conference Proceedings (ICML'14)}, pages
  586--594.

\bibitem[Ben-David et~al., 2003]{ben2003difficulty}
Ben-David, S., Eiron, N., and Long, P.~M. (2003).
\newblock On the difficulty of approximately maximizing agreements.
\newblock {\em Journal of Computer and System Sciences}, 66(3):496--514.

\bibitem[Boucheron et~al., 2005]{boucheron2005theory}
Boucheron, S., Bousquet, O., and Lugosi, G. (2005).
\newblock Theory of classification: A survey of some recent advances.
\newblock {\em ESAIM: Probability and Statistics}, 9:323--375.

\bibitem[Boyd and Vandenberghe, 2004]{boyd2004convex}
Boyd, S. and Vandenberghe, L. (2004).
\newblock {\em Convex Optimization}.
\newblock Cambridge University Press, New York, NY, USA.

\bibitem[Calauz{\`e}nes et~al., 2013]{calauzenes2013calibration}
Calauz{\`e}nes, C., Usunier, N., and Gallinari, P. (2013).
\newblock Calibration and regret bounds for order-preserving surrogate losses
  in learning to rank.
\newblock {\em Machine Learning}, 93(2-3):227--260.

\bibitem[Chen and Sun, 2006]{chen2006consistency}
Chen, D.-R. and Sun, T. (2006).
\newblock Consistency of multiclass empirical risk minimization methods based
  on convex loss.
\newblock {\em Journal of Machine Learning Research}, 7:2435--2447.

\bibitem[Crammer and Singer, 2003]{crammer2003ultraconservative}
Crammer, K. and Singer, Y. (2003).
\newblock Ultraconservative online algorithms for multiclass problems.
\newblock {\em Journal of Machine Learning Research}, 3:951--991.

\bibitem[Do\u{g}an et~al., 2016]{dougan2016unified}
Do\u{g}an, {\"U}., Glasmachers, T., and Igel, C. (2016).
\newblock A unified view on multi-class support vector classification.
\newblock {\em Journal of Machine Learning Research}, 17(45):1--32.

\bibitem[Feldman et~al., 2012]{feldman2012agnostic}
Feldman, V., Guruswami, V., Raghavendra, P., and Wu, Y. (2012).
\newblock Agnostic learning of monomials by halfspaces is hard.
\newblock {\em SIAM Journal on Computing}, 41(6):1558--1590.

\bibitem[Gneiting and Raftery, 2007]{gneiting2007strictly}
Gneiting, T. and Raftery, A.~E. (2007).
\newblock Strictly proper scoring rules, prediction, and estimation.
\newblock {\em Journal of the American Statistical Association},
  102(477):359--378.

\bibitem[Guruprasad and Agarwal, 2012]{guruprasad2012classification}
Guruprasad, H. and Agarwal, S. (2012).
\newblock Classification calibration dimension for general multiclass losses.
\newblock In Bartlett, P.~L., Pereira, F., Burges, C. J.~C., Bottou, L., and
  Weinberger, K.~Q., editors, {\em Advances in Neural Information Processing
  Systems}, volume~25, pages 2087--2095. Curran Associates, Inc.

\bibitem[Hastie et~al., 2009]{hastie2009elements}
Hastie, T., Tibshirani, R., and Friedman, J. (2009).
\newblock {\em The Elements of Statistical Learning: Data Mining, Inference,
  and Prediction, Second Edition}.
\newblock Springer Series in Statistics. Springer.

\bibitem[H{\"o}ffgen et~al., 1995]{hoffgen1995robust}
H{\"o}ffgen, K.-U., Simon, H.-U., and {Van Horn}, K.~S. (1995).
\newblock Robust trainability of single neurons.
\newblock {\em Journal of Computer and System Sciences}, 50:114--125.

\bibitem[Koltchinskii, 2011]{koltchinskii2011oracle}
Koltchinskii, V. (2011).
\newblock {\em Oracle Inequalities in Empirical Risk Minimization and Sparse
  Recovery Problems: {\'E}cole d'{\'E}t{\'e} de Probabilit{\'e}s de Saint-Flour
  XXXVIII-2008}.
\newblock Lecture Notes in Mathematics. Springer Berlin Heidelberg.

\bibitem[Lee et~al., 2004]{lee2004multicategory}
Lee, Y., Lin, Y., and Wahba, G. (2004).
\newblock Multicategory support vector machines: Theory and application to the
  classification of microarray data and satellite radiance data.
\newblock {\em Journal of the American Statistical Association},
  99(465):67--81.

\bibitem[Lin, 2004]{lin2004note}
Lin, Y. (2004).
\newblock A note on margin-based loss functions in classification.
\newblock {\em Statistics \& Probability Letters}, 68(1):73--82.

\bibitem[Liu, 2007]{liu2007fisher}
Liu, Y. (2007).
\newblock Fisher consistency of multicategory support vector machines.
\newblock {\em Proceedings of the Eleventh International Conference on
  Artificial Intelligence and Statistics}, 2:289--296.

\bibitem[Long and Servedio, 2013]{long2013consistency}
Long, P. and Servedio, R. (2013).
\newblock Consistency versus realizable {H}-consistency for multiclass
  classification.
\newblock In Dasgupta, S. and McAllester, D., editors, {\em Proceedings of The
  30\textsuperscript{th} International Conference on Machine Learning},
  volume~28 of {\em JMLR: Workshop \& Conference Proceedings (ICML'13)}, pages
  801--809.

\bibitem[Mammen et~al., 1999]{mammen1999smooth}
Mammen, E., Tsybakov, A.~B., et~al. (1999).
\newblock Smooth discrimination analysis.
\newblock {\em The Annals of Statistics}, 27(6):1808--1829.

\bibitem[Mason et~al., 2000]{mason1999boosting}
Mason, L., Baxter, J., Bartlett, P.~L., and Frean, M.~R. (2000).
\newblock Boosting algorithms as gradient descent.
\newblock In Solla, S.~A., Leen, T.~K., and M{\"{u}}ller, K., editors, {\em
  Advances in Neural Information Processing Systems}, volume~12, pages
  512--518, Cambridge, MA. MIT Press.

\bibitem[Mroueh et~al., 2012]{mroueh2012multiclass}
Mroueh, Y., Poggio, T., Rosasco, L., and Slotine, J.-J. (2012).
\newblock Multiclass learning with simplex coding.
\newblock In Bartlett, P.~L., Pereira, F., Burges, C. J.~C., Bottou, L., and
  Weinberger, K.~Q., editors, {\em Advances in Neural Information Processing
  Systems}, volume~25, pages 2798--2806. Curran Associates, Inc.

\bibitem[Nesterov, 2013]{nesterov2013introductory}
Nesterov, Y. (2013).
\newblock {\em Introductory Lectures on Convex Optimization: A Basic Course}.
\newblock Applied Optimization. Springer US.

\bibitem[Nguyen and Sanner, 2013]{nguyen2013algorithms}
Nguyen, T. and Sanner, S. (2013).
\newblock Algorithms for direct 0--1 loss optimization in binary
  classification.
\newblock In Dasgupta, S. and McAllester, D., editors, {\em Proceedings of The
  30\textsuperscript{th} International Conference on Machine Learning},
  volume~28 of {\em JMLR: Workshop \& Conference Proceedings (ICML'13)}, pages
  1085--1093.

\bibitem[Nock and Nielsen, 2009]{nock2009bregman}
Nock, R. and Nielsen, F. (2009).
\newblock Bregman divergences and surrogates for learning.
\newblock {\em IEEE Transactions on Pattern Analysis and Machine Intelligence},
  31(11):2048--2059.

\bibitem[Ramaswamy and Agarwal, 2016]{ramaswamy2016convex}
Ramaswamy, H.~G. and Agarwal, S. (2016).
\newblock Convex calibration dimension for multiclass loss matrices.
\newblock {\em Journal of Machine Learning Research}, 17(14):1--45.

\bibitem[Ramaswamy et~al., 2013]{ramaswamy2013convex}
Ramaswamy, H.~G., Agarwal, S., and Tewari, A. (2013).
\newblock Convex calibrated surrogates for low-rank loss matrices with
  applications to subset ranking losses.
\newblock In Burges, C.~J.~C., Bottou, L., Welling, M., Ghahramani, Z., and
  Weinberger, K.~Q., editors, {\em Advances in Neural Information Processing
  Systems}, volume~26, pages 1475--1483. Curran Associates, Inc.

\bibitem[Reid and Williamson, 2009]{reid2009surrogate}
Reid, M.~D. and Williamson, R.~C. (2009).
\newblock Surrogate regret bounds for proper losses.
\newblock In Bottou, L. and Littman, M., editors, {\em Proceedings of the
  26\textsuperscript{th} International Conference on Machine Learning},
  ICML'09, pages 897--904, New York, NY, USA. Omnipress.

\bibitem[Reid and Williamson, 2010]{reid2010composite}
Reid, M.~D. and Williamson, R.~C. (2010).
\newblock Composite binary losses.
\newblock {\em Journal of Machine Learning Research}, 11:2387--2422.

\bibitem[Rifkin and Klautau, 2004]{rifkin2004indefense}
Rifkin, R. and Klautau, A. (2004).
\newblock In defense of one-vs-all classification.
\newblock {\em Journal of Machine Learning Research}, 5:101--141.

\bibitem[Shalev-Shwartz and Ben-David, 2014]{shalev2014understanding}
Shalev-Shwartz, S. and Ben-David, S. (2014).
\newblock {\em Understanding Machine Learning: From Theory to Algorithms}.
\newblock Understanding Machine Learning: From Theory to Algorithms. Cambridge
  University Press.

\bibitem[Shi et~al., 2015]{shi2015hybrid}
Shi, Q., Reid, M.~D., Caetano, T., Van~den Hengel, A., and Wang, Z. (2015).
\newblock A hybrid loss for multiclass and structured prediction.
\newblock {\em IEEE Transactions on Pattern Analysis and Machine Intelligence},
  37(1):2--12.

\bibitem[Steinwart, 2007]{steinwart2007how}
Steinwart, I. (2007).
\newblock How to compare different loss functions and their risks.
\newblock {\em Constructive Approximation}, 26(2):225--287.

\bibitem[Steinwart and Christmann, 2008]{steinwart2008support}
Steinwart, I. and Christmann, A. (2008).
\newblock {\em Support vector machines}.
\newblock Springer.

\bibitem[Tewari and Bartlett, 2007]{tewari2007ontheconsistency}
Tewari, A. and Bartlett, P.~L. (2007).
\newblock On the consistency of multiclass classification methods.
\newblock {\em Journal of Machine Learning Research}, 8:1007--1025.

\bibitem[Vapnik, 2013]{vapnik2013nature}
Vapnik, V. (2013).
\newblock {\em The Nature of Statistical Learning Theory}.
\newblock Springer New York.

\bibitem[Weston and Watkins, 1998]{weston1998multiclass}
Weston, J. and Watkins, C. (1998).
\newblock Multi-class support vector machines.
\newblock Technical Report CSD- TR-98-04, Department of Computer Science, Royal
  Holloway College, University of London.

\bibitem[Zhang, 2004]{zhang2004statistical}
Zhang, T. (2004).
\newblock Statistical analysis of some multi-category large margin
  classification methods.
\newblock {\em Journal of Machine Learning Research}, 5:1225--1251.

\bibitem[Zou et~al., 2006]{zou2006margin}
Zou, H., Zhu, J., and Hastie, T. (2006).
\newblock The margin vector, admissible loss and multiclass margin-based
  classifiers.
\newblock Technical report, Technical Report, Statistics Departement, Stanford
  University.

\end{thebibliography}
}

\appendix

\section{Proofs}
\label{app:proofs}

In this section, we present the proofs omitted from the main text, as well as accessory technical results.
We will start with the following two well-known facts.
\begin{fact}
\label{fact:convexLowerBound}
If $\fun{f}{\rl{}}{\rl{}}$ is convex, then for all $a,b \in \rl{}$ $f(a) \geq f(b) + v(a - b)$ for any $v \in \partial f(b)$,
where $\partial f(b)$ is the subdifferential of $f$ at $b$.
\end{fact}
\begin{fact}
\label{fact:infLowerBound}
For $\fun{f}{\rl{n}}{\rl{}}$, and $\sn{X}', \sn{X} \subset \rl{n}$, if for every $x \in \sn{X}$ there exists an $x' \in \sn{X}'$ \st{} $f(x') \leq f(x)$, then $\inf_{x' \in \sn{X}'} f(x') \leq \inf_{x \in \sn{X}} f(x)$.
If $\sn{X}' \subset \sn{X}$, then we have $\inf_{x' \in \sn{X}'} f(x') = \inf_{x \in \sn{X}} f(x)$.
\end{fact}
We will refer to the lower-bound in \cref{fact:convexLowerBound} as a lower-bound for $f(a)$ at $b$.
We will abuse notation and write $f'(b)$ in place of $v$ for $v \in \partial f(b)$ when applying this lower-bound and $f'(b)$ will mean an arbitrary element of $\partial f(b)$.
Where \cref{fact:infLowerBound} is used, we will simply construct the surjective mapping from $\sn{X}$ to $\sn{X}'$.
\propKinkNotCalibrated*
\begin{proof}
Let $L$ be $\sloss{LLW}$ with $\varphi^{\frac{1}{2}\mathrm{-Kink}}$.
We will show that $L$ is not calibrated in a three-class, one-point scenario.
We will carry out a case-by-case proof where we choose a $p \in \Delta_3$ satisfying certain properties, and then we show that the minimizer of $s \mapsto \risk^{\surr}(s,p)$ does not have a unique maximum coordinate corresponding to the unique maximum coordinate of $p$.

Let $p \in \Delta_3$ satisfy the following properties:
\als{
p_1, p_2 &> p_3,\\
2 - 3p_1 - p_2 &> 0,\\
1 - 2p_2 &< 0,\\
2 - 3p_1 - p_2 &> 0,\\
1 - 2p_1 &< 0.
}
We can take, \eg{}, $p = \pb{\frac{8}{20},\frac{7}{20},\frac{5}{20}}$ or $p= \pb{\frac{6}{15},\frac{7}{15},\frac{2}{15}}$.

Consider
\[
	s^* \in \argmin_{s \in \sn{S}_0} \risk^{\surr}(s,p).
\]
Note that $s^* \geq -\one{3}$, otherwise $s^*$ would not be a minimizer of $\risk^{\surr}(s,p)$ within $\sn{S}_0$.
To see that, if $s^*_i < -1$ for some $i$, we can reduce $\risk^{\surr}(s^*,p)$ by increasing $s^*_i$ to $-1$ and decreasing the other coordinates, while we also ensure that $\one{3}\tra s^* = 0$.

If $s^*_3$ is a largest coordinate of $s^*$, $L$ is indeed not calibrated.
Therefore, we assume that $s^*_3 < \max_k s^*_k$.
Because $s^* \geq -\one{3}$ and $\one{3}\tra s^* = 0$, we can only have one $i$ \st{} $s^*_i > \frac{1}{2}$.
In particular, since we assumed that $s^*_3 < \max_k s^*_k$, we must have $s^*_3 < \frac{1}{2}$.

Without loss of generality, we will write $s_3 = s_1 + s_2$, and we will let $s' \eqdef \pb{\frac{1}{2},\frac{1}{2},-1}$.
We will also use the fact that, for any $p \in \Delta_K$, the objective $s \mapsto \risk^{\surr}(s,p)$ is continuous.
Now we will show that under the assumptions above we must have $s^*_1 = s^*_2 = \frac{1}{2}$ by showing that
\begin{enumerate}
\item for all $s \in \sn{S}_0$ \st{} $s_1 > \frac{1}{2}$, $s_2 < \frac{1}{2}$ and $s_1,s_2 > s_3$, we have $\risk^{\surr}(s,p) > \risk^{\surr}(s',p)$,
\item for all $s \in \sn{S}_0$ \st{} $s_1 < \frac{1}{2}$, $s_2 > \frac{1}{2}$ and $s_1,s_2 > s_3$, we have $\risk^{\surr}(s,p) > \risk^{\surr}(s',p)$,
\item for all $s \in \sn{S}_0$ \st{} $s_1 < \frac{1}{2}$, $s_2 < \frac{1}{2}$ and $s_1,s_2 > s_3$, we have $\risk^{\surr}(s,p) > \risk^{\surr}(s',p)$.
\end{enumerate}

For the first case, consider any $s \in \sn{S}_0$ \st{} $s_1 > \frac{1}{2}$, $s_2 < \frac{1}{2}$ and $s_1,s_2 > s_3$, which gives
\als{
	\risk^{\surr}(s,p) &= (1 - p_1)\pb{\frac{1}{2} + 2s_1} + (1 - p_2)(1 + s_2) + (p_1 + p_2)(1 - s_1 - s_2) \\
	&= 2 - \frac{1 - p_1}{2} + (2 - 3p_1 - p_2)s_1 + (1 - 2p_2)s_2.
}
By construction, $2 - 3p_1 - p_2 > 0$ and $1 - 2p_2 < 0$, so if we want to minimize $\risk^{\surr}(s,p)$ we must make $s_1$ as small as possible and $s_2$ as large as possible, which, together with continuity, implies $\risk^{\surr}(s,p) > \risk^{\surr}(s',p)$.

For the second case, consider any $s \in \sn{S}_0$ \st{} $s_1 < \frac{1}{2}$, $s_2 > \frac{1}{2}$ and $s_1,s_2 > s_3$, which gives
\als{
	\risk^{\surr}(s,p) &= (1 - p_1)\pb{1 + s_1} + (1 - p_2)\pb{\frac{1}{2} + 2s_2} + (p_1 + p_2)(1 - s_1 - s_2) \\
	&= 2 - \frac{1 - p_2}{2} + (1 - 2p_1)s_1 + (2 - 3p_1 - p_2)s_2.
}
By construction, $1 - 2p_1 < 0$ and $2 - 3p_1 - p_2 > 0$, so if we want to minimize $\risk^{\surr}(s,p)$ we must make $s_1$ as large as possible and $s_2$ as small as possible, which, together with continuity, implies $\risk^{\surr}(s,p) > \risk^{\surr}(s',p)$
Finally, for any $s \in \sn{S}_0$ \st{} $s_1 < \frac{1}{2}$, $s_2 < \frac{1}{2}$ and $s_1,s_2 > s_3$, we have
\als{
	\risk^{\surr}(s,p) &= (1 - p_1)(1 + s_1) + (1 - p_2)(1 + s_2) + (p_1 + p_2)(1 - s_1 - s_2) \\
	&= 2 + (1 - 2p_1)s_1 + (1 - 2p_2)s_2.
}
By construction, $1 - 2p_1 < 0$ and $1 - 2p_2 < 0$, so if we want to minimize $\risk^{\surr}(s,p)$ we must make $s_1,s_2$ as large as possible, which, together with continuity, implies $\risk^{\surr}(s,p) > \risk^{\surr}(s',p)$.

In conclusion, we have that either $s^*_3 = \max_k s^*_k$ (in which case $L$ is not calibrated) or $s^*_3 < \max_k s^*_k$ and $s^*_1 = s^*_2 = \frac{1}{2}$ (in which case $L$ is not calibrated either).

\end{proof}

\begin{proof}[Proof of \cref{lem:deltaMaxFramework}]
For any $\eps > 0$ and $p \in \yplex$, if $\max_k p_k - \min_k p_k < \eps$ then $\delta(\eps, p) = \infty$, so
\[
    \delta_{\max}(\eps) = \inf_{\substack{p \in \yplex : \\ \max_k p_k - \min_k p_k \geq \eps}} \delta_{\max}(\eps, p).
\]

Now fix $\eps > 0$ and $p \in \yplex$ \st{} $\max_k p_k - \min_k p_k \geq \eps$.
Pick $\jeps \in \sn{J}(\eps,p)$ and $j_0 \in \sn{J}(0,p)$ and let $\bar{p} = \frac{p_{\jeps} + p_{j_0}}{2}$.
We have that
\al{
    \delta_{\max}(\eps, p) &= \inf_{s \in \sn{T}(\sn{S}, \eps, p) }\risk^{\surr}(s,p) - \inf_{s \in \sn{S}} \risk^{\surr}(s,p) \notag \\
    &= \inf_{s \in \sn{M}(\sn{S}, \jeps) \cap \sn{M}(\sn{S}, j_0) }\risk^{\surr}(s,p) - \inf_{s \in \sn{S}} \risk^{\surr}(s,p) \label[equation]{eq:deltaMaxFramework:2} \\
    &= \inf_{s \in \sn{M}(\sn{S}, \jeps) \cap \sn{M}(\sn{S}, j_0) }\sup_{s' \in \sn{S}}\risk^{\surr}(s,p) - \risk^{\surr}(s',p) \label[equation]{eq:deltaMaxFramework:3} \\
    &\geq \inf_{s \in \sn{M}(\sn{S}, \jeps) \cap \sn{M}(\sn{S}, j_0)}\sup_{\substack{s' \in \sn{S}:\\ s'_{k} = s_{k}, k \neq \jeps,j_0}} \risk^{\surr}(s,p) - \risk^{\surr}(s',p) \label[equation]{eq:deltaMaxFramework:4} \\
    &\geq \inf_{s \in \sn{S}}\sup_{s' \in \sn{S}} \risk^{\surr}(s,  (\bar{p}, \bar{p})) - \risk^{\surr}(s', (p_{\jeps}, p_{j_0})) \label[equation]{eq:deltaMaxFramework:6}
}
To obtain \eqref{eq:deltaMaxFramework:2}, we used \cref{cond:infAtJepsJ0}.
In \eqref{eq:deltaMaxFramework:3}, we simply rewrote the objective, and in \eqref{eq:deltaMaxFramework:4} we lower-bounded the supremum by restricting $s'$.
We used \cref{cond:pairing} in \eqref{eq:deltaMaxFramework:6}, where $\sn{S} \subset \rl{2}$ and where $\risk'$ and $\risk^{\surr}$ have domain $\sn{S} \times \simplex{2}$.

If \cref{cond:zetaOfEps} holds, we can combine it with \eqref{eq:deltaMaxFramework:6} to get that
\begin{equation}
    \delta_{\max}(\eps, p) \geq \zeta(p_{j_0} - p_{\jeps}). \label[equation]{eq:deltaMaxFramework:7}
\end{equation}
Since $\delta_{\max}(\eps,p) = \infty$ if $\max_k p_k - \min_k p_k < \eps$, we can use \eqref{eq:deltaMaxFramework:7}, $p_{j_0} - p_{\jeps} \geq \eps$ and $\zeta$ non-decreasing to obtain that $\delta_{\max}(\eps) \geq \zeta(\eps)$, which implies the first statement.

\Cref{cond:infWithP} implies, along with \eqref{eq:deltaMaxFramework:6}, that
\begin{equation}
    \eqref{eq:deltaMaxFramework:6} \geq \inf_{s \in \sn{S}}\sup_{s' \in \sn{S}} \risk^{\surr}(s,  p^0) - \risk^{\surr}(s', p^{\eps}). \label[equation]{eq:deltaMaxFramework:8}
\end{equation}
Moreover, note that for any $p \in \yplex$ we have either \eqref{eq:deltaMaxFramework:8} or $\delta_{\max}(\eps, p) = \infty$, which implies that
\begin{equation}
    \delta_{\max}(\eps) \geq \inf_{s \in \sn{S}}\sup_{s' \in \sn{S}} \risk^{\surr}(s,  p^0) - \risk^{\surr}(s', p^{\eps}). \label[equation]{eq:deltaMaxFramework:9}
\end{equation}
\end{proof}
\propSubOptimalSurrogateRisk*
\begin{proof}
Fix $p \in \rl{2}_+$.
For any $s \in \sn{S}$, by \cref{cond:symmetry}, convexity of $\risk'$ on the first argument and Jensen's inequality,
\als{
    \risk'\pb{s, (\bar{p}, \bar{p})}
    &= \frac{1}{2}\risk'\pb{s,(\bar{p}, \bar{p})} + \frac{1}{2}\risk'\pb{(s_2, s_1), (\bar{p}, \bar{p})} \\
    &\geq \risk'\pb{(\bar{p}, \bar{p}), (\bar{s}, \bar{s})},
}
where $\bar{s} \eqdef \frac{s_1 + s_2}{2}$.
Hence \cref{fact:infLowerBound} implies that
\[
    \inf_{s \in \sn{S}} \risk'\pb{s, (\bar{p}, \bar{p}) } = \inf_{\substack{s \in \sn{S} : \\ s_1 = s_2}} \risk'\pb{s, (\bar{p}, \bar{p})},
\]
which gives the first statement.

Now assume that $p \in \simplex{2}$ and fix $\eps > 0$ \st{} $\sn{T}(\sn{S}, \eps, p) \neq \emptyset$.
By \cref{cond:infAtJepsJ0,cond:symmetry}, respectively,
\als{
    \inf_{s \in \sn{T}(\sn{S}, \eps, p)} \risk^{\surr}(s,p)
    &= \inf_{\substack{s \in \sn{S} : \\ s_1 = s_2}} \risk^{\surr}(s, p) \\
    &= \inf_{\substack{s \in \sn{S} : \\ s_1 = s_2}} \risk^{\surr}(s, \one{2} - p)
}
Fix any $\alpha > 0$ and take $s_{\alpha} \in \sn{S}$ \st{} $(s_{\alpha})_1 = (s_{\alpha})_2$ and \st{}
\[
    \inf_{\substack{s \in \sn{S} : \\ s_1 = s_2}} \risk^{\surr}(s,p) > \risk^{\surr}(s_{\alpha}, p) - \alpha.
\]
Then by \cref{cond:symmetry}
\[
     \risk^{\surr}(s_{\alpha}, p) =  \risk^{\surr}(s_{\alpha}, \one{2} -  p),
\]
and by linearity of expectation,
\[
    \frac{1}{2}\risk^{\surr}(s_{\alpha}, p) + \frac{1}{2}\risk^{\surr}(s_{\alpha}, \one{2} - p) = \frac{1}{2}\risk^{\surr}\pb{s_{\alpha}, \pb{\frac{1}{2}, \frac{1}{2}}}.
\]
Hence,
\als{
    \inf_{s \in \sn{S}} \risk^{\surr}\pb{s, \pb{\frac{1}{2}, \frac{1}{2}}}
    &= \inf_{\substack{s \in \sn{S} : \\ s_1 = s_2}} \risk^{\surr}\pb{s, \pb{\frac{1}{2}, \frac{1}{2}}} \\
    &= \inf_{\substack{s \in \sn{S} : \\ s_1 = s_2}} \frac{1}{2}\risk^{\surr}(s, p) + \frac{1}{2} \risk^{\surr}(s, \one{2} - p) \\
    &\geq \frac{1}{2}\inf_{\substack{s \in \sn{S} : \\ s_1 = s_2}} \risk^{\surr}(s, p) + \frac{1}{2}\inf_{\substack{s \in \sn{S} : \\ s_1 = s_2}} \risk^{\surr}(s, \one{2} - p) \\
    &= \inf_{\substack{s \in \sn{S} : \\ s_1 = s_2}} \risk^{\surr}(s, p) \\
    &> \risk^{\surr}(s_{\alpha}, p) - \alpha \\
    &= \risk^{\surr}\pb{s_{\alpha}, \pb{\frac{1}{2}, \frac{1}{2}}} - \alpha, \\
    &\geq \inf_{s \in \sn{S}} \risk^{\surr}\pb{s, \pb{\frac{1}{2}, \frac{1}{2}}} - \alpha.
}
which implies the second result, since the above holds for any $\alpha > 0$.
\end{proof}
\lemmaBinaryTightness*

\begin{proof}
Since \cref{cond:pairing} holds trivially when $\K = 2$, we can use \cref{lem:deltaMaxFramework} to obtain that
\[
    \delta_{\max}(\eps) \geq \delta_{\mathrm{binary}}(\eps).
\]
The result follows by using \cref{prop:subOptimalSurrogateRisk} to show that
\[
    \delta_{\max}(\eps, p^{\eps}) = \delta_{\mathrm{binary}}(\eps)
\]
and then recalling that $\delta_{\max}(\eps, p^{\eps}) \geq \delta_{\max}(\eps)$.
\end{proof}
\begin{restatable}{proposition}{propPEps}
\label{prop:pEps}
For any $\eps > 0$ and $p_1,p_2$ \st{} $0 \leq p_2 \leq p_1 - \eps$ and $p_1 + p_2 \leq 1$,
we have
\[
\frac{p_1}{p_1 + p_2} \geq \frac{1}{2} + \frac{\eps}{2} \cdot \frac{1}{p_1 + p_2} \geq \frac{1 + \eps}{2}.
\]
\end{restatable}
\begin{proof}
We have
\als{
\frac{p_1}{p_1 + p_2}
&\geq \frac{1}{2} \cdot \frac{p_1}{p_1 + p_2} + \frac{1}{2} \cdot \frac{p_2}{p_1 + p_2} + \frac{\eps}{2} \cdot \frac{1}{p_1 + p_2} \\
&= \frac{1}{2} + \frac{\eps}{2} \cdot \frac{1}{p_1 + p_2} \\
&\geq \frac{1}{2} + \frac{\eps}{2}~,
}
where for the first inequality we used that $p_1 \geq p_2 + \eps$ and for the second we used that $p_1 + p_2 \leq 1$.
\end{proof}
\begin{restatable}{proposition}{propLogisticInf}
\label{prop:logisticInf}
For any $p_1,p_2 \geq 0$,
\[
    \inf_{s \in \simplex{2}} -p_1 \ln s_1 + -p_2 \ln s_2 = -p_1 \ln \frac{p_1}{p_1 + p_2} - p_2 \ln \frac{p_2}{p_1 + p_2}.
\]
\end{restatable}

\begin{proof}
We have
\[
    \inf_{s \in \simplex{2}} -p_1 \ln s_1 - p_2 \ln s_2 = (p_1 + p_2) \inf_{s \in \simplex{2}} -q \ln s_1 + -(1 - q) \ln s_2,
\]
where $q \eqdef \frac{p_1}{p_1 + p_2}$.
If $q \in \cb{0,1}$, the infimum is taken at $s = \pb{q, 1 - q}$.
Otherwise, the infimum above is taken at $s = (s^*, 1 - s^*)$ where $s^*$ satisfies $s^* \in (0,1)$ and
\[
    \frac{q}{s^*} - \frac{1 - q}{1 - s^*} = 0,
\]
that is, $s^* = q$, which gives the result.
\end{proof}

\end{document}